\documentclass{article}
\usepackage{amsfonts,amsmath,amssymb,amsthm,bm}
\usepackage{mathtools}
\usepackage{ mathrsfs }

\newcommand{\defn}[1]{\textbf{#1}}

\newcommand{\defeq}{\mathrel{\stackrel{\textnormal{\tiny def}}{=}}}

\def\eqref#1{equation~\ref{#1}}

\def\1{\bm{1}}

\DeclareMathAlphabet{\mathsfit}{\encodingdefault}{\sfdefault}{m}{sl}
\SetMathAlphabet{\mathsfit}{bold}{\encodingdefault}{\sfdefault}{bx}{n}

\newtheorem{theorem}{Theorem}

\newcommand{\ent}{\mathcal H}

\usepackage{iclr2027_conference}
\makeatletter\providecommand\new@fontshape{}\makeatother
\usepackage{gb4e}
\noautomath
\newcommand{\zh}[1]{\begin{CJK*}{UTF8}{gbsn}#1\end{CJK*}}
\usepackage{times}
\usepackage{latexsym}
\usepackage{CJKutf8}

\usepackage[T1]{fontenc}

\usepackage[utf8]{inputenc}

\usepackage{microtype}

\usepackage{inconsolata}

\usepackage{graphicx}
\usepackage{wrapfig}
\usepackage{hyperref}
\usepackage{bbold}
\usepackage{booktabs}
\usepackage{enumitem}
\usepackage{cleveref}
\crefname{xnumi}{example}{examples}
\Crefname{xnumi}{Example}{Examples}
\crefname{xnumii}{example}{examples}
\Crefname{xnumii}{Example}{Examples}
\creflabelformat{xnumi}{(#2#1#3)}
\creflabelformat{xnumii}{(#2#1#3)}
\crefformat{section}{\S#2#1#3}
\Crefformat{section}{\S#2#1#3}
\crefformat{subsection}{\S#2#1#3}
\Crefformat{subsection}{\S#2#1#3}
\crefformat{subsubsection}{\S#2#1#3}
\Crefformat{subsubsection}{\S#2#1#3}
\crefrangeformat{section}{\S\S#3#1#4--#5#2#6}
\Crefrangeformat{section}{\S\S#3#1#4--#5#2#6}
\usepackage{ amssymb }
\usepackage{forest}
\usepackage{randtext}

\usepackage{todonotes}
\usepackage{amsmath,xcolor}
\definecolor{mintgreen}{RGB}{152, 255, 152}
\makeatletter
\newcommand*\iftodonotes{\if@todonotes@disabled\expandafter\@secondoftwo\else\expandafter\@firstoftwo\fi}  %
\makeatother

\definecolor{claudecol}{RGB}{182, 90, 50}

\definecolor{linkpurple}{HTML}{551A8B}
\hypersetup{colorlinks=true, allcolors=linkpurple}

\setcitestyle{authoryear,open={(},close={)}} 
\newcommand{\citeposs}[1]{\citeauthor{#1}'s \citeyearpar{#1}}

\renewcommand{\ent}{\mathrm{H}}
\newcommand{\MI}{\mathrm{I}}
\newcommand{\Prob}{\ensuremath{\mathbb{P}}}

\newcommand{\metric}{\ensuremath{d}}

\definecolor{colM}{HTML}{1F77B4}
\definecolor{colJ}{HTML}{3446eb}
\definecolor{colC}{HTML}{D62728}
\definecolor{colU}{HTML}{2CA02C}
\definecolor{colZ}{HTML}{9467BD}
\newcommand{\rvM}{\ensuremath{{\color{colM}M}}}
\newcommand{\rvJ}{\ensuremath{{\color{colJ}J}_\epsilon}}
\newcommand{\rvC}{\ensuremath{{\color{colC}C}}}
\newcommand{\rvU}{\ensuremath{{\color{colU}U}}}
\newcommand{\g}{\ensuremath{{\color{colZ}g}}}
\newcommand{\f}{\ensuremath{{\color{colM}f}}}

\newcommand{\setM}{\ensuremath{{\color{colM}\mathcal{M}}}}
\newcommand{\setC}{\ensuremath{{\color{colC}\mathcal{C}}}}
\newcommand{\setU}{\ensuremath{{\color{colU}\mathcal{U}}}}
\newcommand{\insM}{\ensuremath{{\color{colM}m}}}
\newcommand{\insC}{\ensuremath{{\color{colC}c}}}
\newcommand{\insU}{\ensuremath{{\color{colU}u}}}
\newcommand{\rvZ}{\ensuremath{{\color{colZ}Z}}}
\newcommand{\setZ}{\ensuremath{{\color{colZ}\mathcal{Z}}}}
\newcommand{\insZ}{\ensuremath{{\color{colZ}z}}}

\newtheorem{corollary}{Corollary}[theorem]
\newtheorem{lemma}{Lemma}

\usepackage{algorithm}
\usepackage{algpseudocode}
\usepackage{tikz}
\usetikzlibrary{shapes, positioning}
\usepackage{tikz-3dplot}
\usepackage{pgfplots}
\pgfplotsset{compat=1.18}
\usetikzlibrary{arrows.meta}
\usepackage{placeins}
\title{A Formal Limitation on Learning Human Language From Textual Corpora}

\author{Emily Cheng \\
  Universitat Pompeu Fabra \\
  \href{mailto:emilyshana.cheng@upf.edu}{\texttt{emilyshana.cheng@upf.edu}} \\\And
  Ryan Cotterell \\
  ETH Z\"urich \\
  \href{mailto:ryan.cotterell@inf.ethz.ch}{\texttt{ryan.cotterell@inf.ethz.ch}} \\}

\iclrfinalcopy 
\begin{document}
\maketitle
\begin{abstract}
Can a listener recover what a speaker means from the form of an utterance alone? We answer this question information-theoretically, and for a listener given by \emph{any} featurizer of text, including the hidden states of contemporary large language models. Modeling language use as a joint distribution over meanings, contexts, and utterances, we derive upper bounds on the probability that a decoder recovers a speaker's intended meaning from a representation of the utterance. The bounds are governed by the uncertainty that form leaves about meaning, which splits into an irreducible part and a part that only (extralinguistic) context, but never the utterance alone, can resolve. Because these quantities are intrinsic to language, no representation, however much text or supervision produced it, can surpass them. The bounds apply, moreover, to meaning spaces that are discrete or continuous. We provide empirical evidence in support of the theory through experiments on artificial languages, Mandarin zero-pronoun resolution, and color reference.
\end{abstract}

\makeatletter
\let\oldaddcontentsline\addcontentsline
\renewcommand{\addcontentsline}[3]{}
\makeatother

\section{Introduction}
How much about natural language can be learned from text alone? It is an old question, and for over half a century the classical answer, from formal language learning theory, was largely negative. 
Famously, \citet{gold1967language} showed that a learner presented with an unbounded stream of grammatical strings, but never told which strings are \emph{un}grammatical, cannot in general be guaranteed to converge on the target grammar. 
\citet{angluin1980inductive} then tightened the analysis, determining exactly which classes of languages \emph{are} identifiable from positive data---a condition that the expressive grammars posited for natural language fail to meet.
Read literally, these results cast text, an unlabeled stream of positive examples, as a fundamentally impoverished signal from which to acquire a language.

Probability appeared to offer a way out. Gold's setting demanded the exact identification of 
the entirety of 
a boolean grammar: a string is either in the language or not.
However, learning from positive data \emph{is} conceivable if the problem of grammar induction is relaxed to a probabilistic setting where the goal is no longer to learn a strict binary distinction between grammatical and agrammatical strings, but rather the probability distribution over strings. This reorientation has deep roots in \citeposs{harris1954distributional} distributional hypothesis, that a linguistic unit is characterized by the contexts in which it occurs \citep{Wittgenstein1953-WITPI-4,firth1957synopsis}. In modern natural language processing, the distributional hypothesis is operationalized through distributed representations \citep{hinton1986learning,bengio2003neural,mikolov2013efficient,levy2014neural}, in which words and their contexts are encoded as real-valued vectors whose geometry reflects distributional similarity. Language models built on such representations estimate distributions over strings directly, and large language models (LLMs) trained on trillions of tokens of text continue in this tradition \citep[][\emph{inter alia}]{grattafiori2024llama3herdmodels,groeneveld2024olmoacceleratingsciencelanguage,yang2025qwen3technicalreport}.
Distributional methods of this kind show that a remarkable amount about linguistic form, and even about meaning \citep{Boleda_2020}, can be recovered from the statistical structure of text alone.

In the probabilistic setting, not everyone is convinced that form suffices. \citet{bender-koller-2020-climbing} pushed against this optimism with an influential thought experiment: an octopus that eavesdrops only on the linguistic forms exchanged between two speakers across an undersea cable, never observing the world they share, cannot, Bender and Koller argue, learn meaning, which is grounded in communicative intent and the extralinguistic world \citep{harnad1990symbol}. In some respects this skepticism has proven prescient, as the most capable LLMs are no longer trained on text alone. A post-training stage now routinely augments the maximum-likelihood estimation of a text distribution with extralinguistic supervision; most prominently, reinforcement learning from human feedback (RLHF) optimizes models against human preference judgments \citep{christiano2017deep,ouyang2022training}, injecting a signal about communicative success that raw corpora do not contain. As the frontier continues to move beyond the purely textual, the original question grows sharper rather than moot. We pose it in its probabilistic form: given the distribution of language use, what are the in-principle limits on what \emph{any} system---however much text it observes---can infer about language, and in particular about the \emph{meanings} that speakers intend, from form alone?

Before answering this question, we first remark that \emph{meaning} admits many conceptualizations. For example, meaning can be understood in terms of \emph{reference} (the entities that the utterance picks out in the world); or \emph{sense} (the ``thought'' an utterance evokes, independent of reference) \citep{Frege1892-FREBSU}. In addition, meaning can be analyzed \emph{conventionally}---on utterances in-and-of-themselves---or \emph{pragmatically}, taking extralinguistic context into account \citep{Grice_1957,Sperber1995-SPER}. \citet{Wittgenstein1953-WITPI-4} provides an example of this distinction, in which a builder shouts ``Slab!'' to his coworkers on a construction site: the \emph{conventional} meaning of the utterance is (simply) a slab. However, in context, the speaker's intended (or \emph{pragmatic}) meaning is instead ``Bring me the stone slab.'' Human language use is fundamentally pragmatic \citep{Yule_1996}. Indeed, every (human-produced) natural language utterance---including those in an LLM's training data---was said in some social, physical, and cultural context in which a speaker \emph{intended a meaning} and, informed by context, produced an utterance from which a listener could recover that meaning \citep{Grice1975-GRILAC-6}. Because speaker intent is so key to communication between a speaker and a listener, then critically, speaker intents also underlie the \emph{distribution} of linguistic forms at scale, which aggregates over these single instances of communication. Despite how foundational speaker intent is to language use, whether systems trained on the distribution of forms alone can represent intended meanings has been understudied in comparison to (conventional) referential or sentential notions of meaning \citep{Merrill_Goldberg_Schwartz_Smith_2021,mandelkern-linzen-2024-language,Baggio_Murphy_2026}. For this reason, we privilege \emph{intended meaning} in our analysis of meaning representation. Our analysis reveals that systems trained on forms (divorced from context) learn provably impoverished representations of intended meaning.

\section{A Model of Communication}
\label{sec:setup}

Our theoretical results in \Cref{sec:theory} rest on a model of (human) communication in context. 
In this section, we present the model, which is based loosely on Shannon's communication channel \citep{shannon1948mathematical} and formalizes communication between a speaker and a listener.

\paragraph{Notation.}
We denote random variables in uppercase ($X$), their realizations in lowercase ($x$), and their supports in calligraphic type ($\mathcal{X}$); thus $x \in \mathcal{X}$ is one particular value among the set $\mathcal{X}$ of possible values.
Our model of communication involves the following random variables:
\begin{itemize}[leftmargin=*]
    \item $\rvM$: the \defn{meaning} the speaker intends. Its support $\setM$ may be finite, e.g., a closed set of pronouns, or continuous, e.g., a space of colors. To measure distances, we equip $\setM$ with a metric $\metric$ (\Cref{sec:continuous}).\looseness=-1
    \item $\rvC$: the \defn{context} in which communication takes place, encompassing everything other than the current utterance, both extralinguistic, e.g., speaker identities, and linguistic, e.g., the prior conversation history. We assume the context is shared by speaker and listener.\looseness=-1
    \item $\rvU$: the \defn{utterance}: the linguistic form the speaker produces to convey $\rvM$ in context $\rvC$.\looseness=-1
    \item $\widehat{\rvM}$: the listener's \defn{inferred meaning}. 
    $\widehat{\rvM}$ shares the support $\setM$ with $\rvM$, and communication succeeds to the extent that $\widehat{\rvM}$ matches $\rvM$.\looseness=-1
    \item $\rvZ$: a \defn{representation} of the utterance, e.g., the activations of an LLM run on $\rvU$. Representations are produced by a featurizer $\g \colon \setU \to \setZ$ and mapped back to meanings by a decoder $\f \colon \setZ \to \setM$ (\Cref{sec:theory}).\looseness=-1
\end{itemize}
We write $\Prob$ for a probability measure over the product space of all random variables.
We write $\ent(\rvM)$ for the Shannon entropy of $\rvM$, $\ent(\rvM \mid \rvU)$ for the conditional entropy of $\rvM$ given $\rvU$, and $\MI(\rvM;\rvU) = \ent(\rvM) - \ent(\rvM \mid \rvU)$ for their mutual information; conditional variants such as $\MI(\rvM;\rvC \mid \rvU)$ are defined analogously.

\begin{figure}
    \centering
    \begin{minipage}[c]{0.48\linewidth}
        \centering
        \resizebox{\linewidth}{!}{%
\providecommand{\rvM}{M}\providecommand{\rvC}{C}\providecommand{\rvU}{U}
\begin{tikzpicture}[
    node distance=2.2cm,
    every node/.style={draw, circle, minimum size=1cm},
    >=stealth
]

\node (M) {$\rvM$};
\node (U) [right of=M, xshift=-0.3cm] {$\rvU$};
\node (Mhat) [right of=U, xshift=1.5cm] {$\widehat{\rvM}$};

\node (C) [above of=U, yshift=-0.5cm] {$\rvC$};

\draw[->, semithick] (C) -- (M);
\draw[->, semithick] (M) -- (U);
\draw[->, semithick] (C) -- (U);

\draw[->, semithick] (U) -- (Mhat);
\draw[->, semithick] (C) -- (Mhat);


\draw[dashed, rounded corners]
  ($(M.north west)+(-0.4,0.6)$) rectangle
  ($(M.south east)+(2.3,-0.6)$);
\node[draw=none] at ($(M.south)+(-0.1,-0.25)$) {\small speaker};

\draw[dashed, rounded corners]
  ($(Mhat.north west)+(-0.7,0.6)$) rectangle
  ($(Mhat.south east)+(0.7,-0.6)$);

\node[draw=none] at ($(Mhat.south)+(0.35,-0.25)$) {\small listener};
\end{tikzpicture}
        }
    \end{minipage}
    \hfill
    \begin{minipage}[c]{0.48\linewidth}
        \caption{\small \textbf{Communication model.}
        During communication, a context $\rvC$ influences a speaker's intended
        meaning, given by random variable $\rvM$, which the speaker encodes
        into an utterance $\rvU$ given that same context. Under the same
        context $\rvC$, the listener decodes the utterance $\rvU$ into their
        inferred meaning $\widehat{\rvM}$. Communication is
        considered successful if $\rvM \approx \widehat{\rvM}$.}
        \label{fig:graphical_model}
    \end{minipage}
\end{figure}
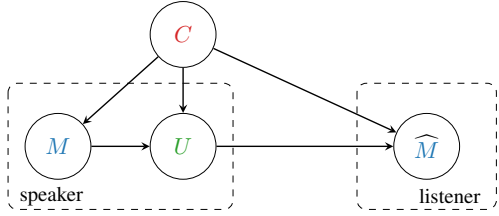

\paragraph{The communication model.}
The graphical model in \Cref{fig:graphical_model} wires together the variables we just defined. In particular, \Cref{fig:graphical_model} realizes a dyadic model of communication based on Shannon's communication channel \citep{shannon1948mathematical}. During a single dialogue turn, a speaker (left dashed box) intends a meaning $\rvM$ in context $\rvC$ and produces an utterance $\rvU$ modulated by the context, where the context can contain both \emph{extralinguistic} factors, such as speaker identities, and \emph{linguistic} ones, such as the prior conversation history \citep{Grice1975-GRILAC-6,Yule_1996}. 
Given the same context as the speaker, the listener (right dashed box) reconstructs the meaning of the utterance.

We model language use in a speech community as made up of many such single instances of dyadic communication, each between a single speaker and listener. Language use can therefore be described by a \emph{distribution} over utterances, contexts, and meanings produced by various speaker--listener pairs \citep{rsa_frank_goodman}. We are interested in information-theoretic properties of the entire system, so we abstract away from single communicative instances, modeling their components as random variables and communication as a generative process over them.
The graphical model in \Cref{fig:graphical_model} encodes the following relationships in its edges:
\begin{itemize}[leftmargin=*]
    \item $\rvC \to \rvM$: meanings arise \emph{in} context; what a speaker intends to convey may depend on the situation, including the topic under discussion, the interlocutors, and their goals.
    \item $\rvM \to \rvU$ and $\rvC \to \rvU$: the speaker chooses an utterance form given both their intended meaning and the context, e.g., leaving unsaid what the context already makes clear.
    \item $\rvU \to \widehat{\rvM}$ and $\rvC \to \widehat{\rvM}$: the listener forms their inference from the utterance they perceive and the context they share with the speaker.
\end{itemize}

The model encodes two key assumptions. First, we assume that the same context is available to the speaker and listener---in practice, this is often not the case: when the interlocutors' contexts differ, they may undergo a process of pragmatic repair
to establish a common ground \citep{Clark_Brennan_1991}. 
Indeed, we do not analyze the multi-turn case at all, focusing instead on single-turn one-way communication from speaker to listener. Second, the linguistic channel may contain noise. For simplicity, our exposition assumes the channel is noiseless; however, our main results directly apply to the noisy channel setting, which we will revisit briefly in \Cref{sec:theory}.


\section{Bounding meaning inference from text-based representations}
\label{sec:theory}
Now given our model of communication (\Cref{fig:graphical_model}), we set up the problem of intended meaning representation in a listener who only accesses utterance forms. While our exposition focuses on language models---importantly, text-only LLMs that do not observe context---the following applies to all systems that only observe ungrounded forms. To make precise our notion of meaning representation, here, we first address how we operationalize meanings, followed by representation.

We model meanings as a metric space $(\setM, \metric)$ \citep{Gardenfors_2000}. Our theory handles both finite meaning sets $\setM$, e.g., the closed set of pronouns, or infinite meaning sets, such as continuous color spaces or the natural numbers. Meanings $\setM$ can, moreover, be bounded, as in the case of colors (below, \textbf{(a)}), which are constrained by human visual perception; or unbounded, as in the case of the natural numbers (below, \textbf{(b)}). 

In \textbf{(b)}, even though the natural numbers are unbounded, humans rarely communicate about extremely large
numbers. Instead, cross-linguistically, the likelihood of communicating about a natural number, described by $\Prob(\rvM)$ in our notation, empirically falls off according to an inverse-square law \citep{Dehaene_Mehler_1992,Piantadosi_2014}. Then, critically, $\Prob(\rvM)$
is \emph{tight}, that is, for every small number $\delta > 0$, there exists a compact (closed and bounded) subset of meanings $\setM' \subseteq \setM$ that make up most of the probability mass, i.e., such that $\Prob(\rvM \in \setM') > 1-\delta$. We assume that in language use, meaning distributions are tight, behaving generally like the natural numbers. To keep the presentation simple, in the main text we will present only the \emph{bounded} meanings case, which already meets the tightness assumption (\Cref{sec:discrete_theory}, \ref{sec:continuous}), leaving the more complicated unbounded case to the Appendix.  
\begin{wrapfigure}{l}{0.32\linewidth}
    \centering
    \vspace{0ex}

    \includegraphics[
        trim={3mm 10mm 3mm 30mm},
        clip,
        width=0.7\linewidth
    ]{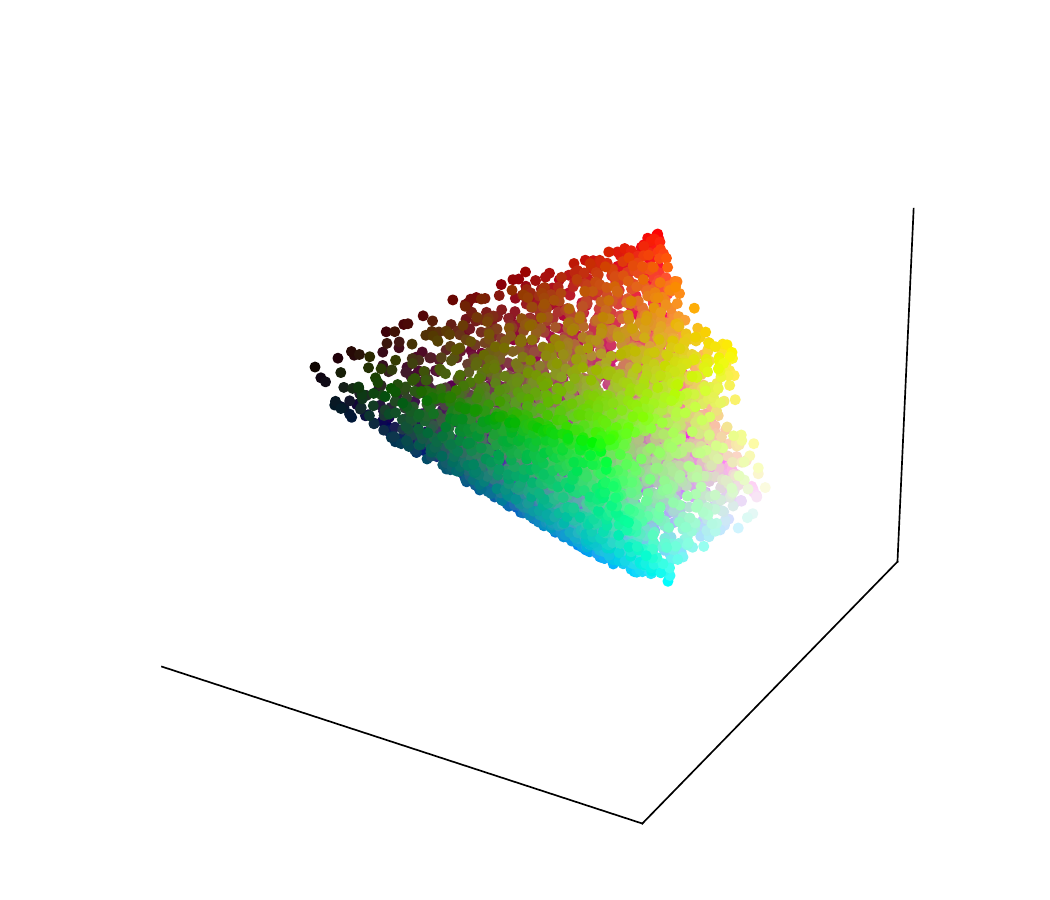}

    \vspace{0.25ex}
    \footnotesize
    (a) Bounded $\setM$ (colors)

    \vspace{4ex}

    \begin{tikzpicture}[
    >=Latex,
    every node/.style={font=\small}
]

\draw[-{Latex[length=3mm,width=2mm]}, line width=0.8pt]
    (1,0) -- (3.7,0);

\foreach \x in {1,...,3}{
    \fill (\x,0) circle (1.2pt);
    \node[below=2pt] at (\x,0) {$\x$};
}

\end{tikzpicture}

    \vspace{0.25ex}
    \footnotesize
    (b) Unbounded $\setM = \mathbb N$ \\
    \vspace{2ex}
\end{wrapfigure}

\vspace{-4ex}
\paragraph{Learning to represent meaning.}
What does it mean for a system like an LLM to (learn to) represent speaker meaning? 
Language models take linguistic forms as input and are trained to output a distribution over the next token. As a byproduct of this training, these models learn a set of feature maps on linguistic forms, given by the models' internal activations under a textual input. 
In our paper's notation, LLM activations implement a function $\g \colon \setU \to \setZ$, where $\g$ maps the space of \emph{forms} ($\setU$) to a \emph{representation space} ($\setZ$, typically $\mathbb R^d$). We propose that a necessary condition for LLM activations to \emph{represent} speaker meaning is that the representation space $\setZ$ permits successful inference of the original meaning. That is, there should exist some $\f \colon \setZ \to \setM$ that approximately maps representations in $\setZ$ back to true meanings in $\setM$.

The problem of learning a good meaning representation can be distilled into learning a $\g: \setU \to \setZ$ such that approximate meaning recovery from $\setZ$ is possible: that is, such that there exists $\f :\setZ \to \setM$ where $\metric(\f(\g(\insU)), \insM) \leq \epsilon$. We consider two cases in particular: $\epsilon=0$, corresponding to exact recovery, or \emph{classification}, and $\epsilon > 0$, corresponding to $\epsilon$-accurate recovery in a \emph{regression} setting; the former is only applicable for finite $\setM$. In the following sections, for each case, we will upper-bound the probability of successful meaning recovery from a representation of the utterance (\Cref{thm:classification,coro:regression}).
We emphasize that our definitions (and, by extension, our theoretical results) apply to \emph{any} text featurizer $\g$, including the internal activations of a perfect learner trained on infinite data. In short, our results put forth a theoretical limit on meaning representation from text. Relevant to a machine learning audience, this limit also constrains the extent to which meanings can be \emph{learned} from text: what one cannot represent, one cannot learn. 

\paragraph{Meaning in utterance vs.~context} 
To build intuition about meaning representation in a purely text-based system, first note that both \underline{context} and \underline{utterance} contribute to meaning inference in communication.
The listener's overall ability to infer meaning can be described by the conditional entropy $\ent(\rvM\mid \rvU, \rvC)$ of meanings $\rvM$ given utterances $\rvU$ and context $\rvC$: if $\ent(\rvM\mid \rvU, \rvC)\approx 0$, then it is generally possible to recover the meaning in context. But what of intended meaning can be recovered from utterances alone? Here, our quantity of interest is the conditional entropy $\ent(\rvM \mid \rvU)$ of meanings given utterances (without context), which can be decomposed as follows:
\begin{equation}
\label{eq:mi_breakdown}
\ent(\rvM\mid \rvU) = \ent(\rvM\mid \rvU,\rvC) + \MI(\rvM;\rvC\mid \rvU).
\end{equation}
The RHS consists of $\ent(\rvM\mid \rvU,\rvC)$: irreducible meaning ambiguity in communication; and $\MI(\rvM;\rvC\mid \rvU)$: what context contributes to meaning independently of the utterance. We assume that in natural language, $\ent(\rvM \mid \rvU, \rvC)$ is generally small---communication is typically understood. Then, remaining meaning ambiguity given the utterance is explained by context, seen by the term $\MI(\rvM;\rvC\mid \rvU)$.  
Next, in \Cref{thm:classification,coro:regression}, we will see how the terms in \Cref{eq:mi_breakdown} govern how well \emph{any} representation of a text utterance, e.g., its activations under an LLM, represent speaker meaning. 

\subsection{Exact recovery for finite meaning spaces}
\label{sec:discrete_theory}

We consider a classification task that aims to recover the true meaning $\rvM = \widehat \rvM$ from a representation of the utterance. 
We assume the random variable $\rvM$ over meanings has finite categorical support $\setM$ (we analyze the countably infinite case in \Cref{app:unbounded_countable} and the uncountable case in the next section). The goal in this task is to find some classifier $\f: \setZ \to \setM$ that maps representations $\rvZ$ to their correct meanings. 
Then, using \Cref{eq:mi_breakdown}, we can upper-bound the probability that \emph{any} $\f $ is correct:

\begin{theorem}
\label{thm:classification}
    Let $\g \colon \setU \to \setZ$ be a map from utterances ($\setU$) to representations ($\setZ$). Let $\f\colon \setZ \to \setM$ map representations to the original meanings ($\setM$). Let the probability that $\f$ is correct be $p_e := \Prob[\f(\g (\rvU)) = \rvM]$. Then,
        \begin{align}
        \sup_{\f \in \setM^{\setZ}} p_e &\leq 1 - \frac{\ent(\rvM 
        \mid \rvU)-\ent_2(p_e)}{\log(|\setM|-1)} 
        =
        1- \frac{\overbrace{\ent(\rvM \mid \rvU,\rvC)}^{\text{irreducible ambiguity}} + \overbrace{\MI(\rvM;\rvC\mid \rvU)}^{\text{meaning from context only}}-\ent_2(p_e)}{\log(|\setM|-1)}. \nonumber 
    \end{align}
\end{theorem}
\begin{proof}
    See \Cref{app:more_theory}.
\end{proof}


\noindent \Cref{thm:classification} bounds the probability of correct classification, $p_e \defeq \Prob[\f (\g (\rvU)) = \rvM]$, when the meaning space $\setM$ is finite and categorical.\footnote{If the linguistic channel is noisy, then the inequality is strict.} 
While categorical meaning spaces are common (e.g., kinship terms \citep{Kemp_Regier_2012}), many meaning spaces of interest may however be \emph{uncountably infinite}---for instance, the continuous space of colors. In the next section, we derive an analogous result for a \emph{regression} setting on potentially uncountable spaces.

\subsection{Approximate recovery for infinite meaning spaces}
\label{sec:continuous}

\Cref{thm:classification} naturally extends to a regression setting over a potentially uncountable meaning space. Continuous meaning spaces fall into this category. Here, instead of considering the probability that $\f $ exactly matches the target meaning, under a regression objective we let $p_e$ be the likelihood that $\f (\rvZ)$'s meaning estimate is strictly within $\epsilon$ of the ground truth, or $\metric(\widehat \rvM, \rvM) < \epsilon$. 

Reformulating the problem to the setting of $\epsilon$-accurate recovery requires some extra considerations. Notably, the proof of \Cref{thm:classification} relied on Fano's inequality, which applies to random variables ($\rvM, \widehat \rvM$) over \emph{finite} alphabets, so it cannot be used here for uncountably infinite meaning spaces. The workaround is to reduce the regression problem to a \emph{classification problem} where Fano's inequality can be directly applied. Formally, let the binary event $A$ be the event that $\epsilon$-accurate recovery holds, i.e., $A = \mathbb 1[\metric(\widehat \rvM, \rvM) < \epsilon]$. We want to construct the event $B$ that some classification problem, derived from $A$, is solved. Further, we want to design event $B$ such that $A\Rightarrow B$; then, $\Prob(A) \leq \Prob(B)$. That way, we can upper-bound $\Prob(B)$
using Fano's inequality in the setting of \Cref{thm:classification}, which will bound the likelihood of $\epsilon$-accurate recovery, $\Prob(A)$, by extension.

We now construct the event $B$. Let $\{\insM_1, \dots, \insM_{N(\epsilon)}\}$ be a maximal $2\epsilon$-packing of $\setM$. That is, $\metric(\insM_i, \insM_j) \geq 2\epsilon$ for all $i\neq j \in 1\dots N(\epsilon)$, where $N(\epsilon)$ is the packing number. We set up a discrete random variable $\rvJ$ over these packing points, where $\rvJ$ is being classified in the setup of \Cref{thm:classification}. 
To set up the distribution of $\rvJ$, we make use of a ``quantizer'' function $q: \setM \to \{\insM_n\}_{n=1}^{N(\epsilon)}$ that maps points in $\setM$ to the nearest packing point in $\{\insM_n\}_{n=1}^{N(\epsilon)}$, where $\rvJ \defeq q(\rvM)$.

Now, let $B$ be the event that packing points are correctly classified given the corresponding utterance, i.e., $B = \mathbb 1[\widehat \rvJ = \rvJ]$, where $\widehat \rvJ$ is the estimated packing point. By definition of $2\epsilon$-packing, event $A\Rightarrow B$. That is, if for every $\insM \in \setM$ the decoder's estimate is less than $\epsilon$, or $\metric(\f(\g(\insU)), \insM)< \epsilon$, then we are guaranteed to correctly classify each packing point by taking the closest one to the decoder's output. Since $A\Rightarrow B$, then as desired, $\Prob(A) \leq \Prob(B)$.

Now, we can upper-bound $\Prob(A)$---the likelihood of $\epsilon$-accurate recovery---via an upper bound on $\Prob(B)$. We assume that $\setM$ is bounded (for unbounded $\setM$, see \Cref{app:unbounded_uncountable}). This yields the theorem:

\begin{theorem}[$\epsilon$-accurate regression]
\label{coro:regression}
    Let $\g: \setU \to \setZ$ be a map from utterances ($\setU$) to representations ($\setZ$). Let $(\setM, \metric)$ be a bounded metric space and $N(\epsilon)$ be its maximal $2\epsilon$-packing number. Abbreviate the probability that $\f$ is $\epsilon$-accurate as $p_e^\epsilon = \Prob[\metric(\f(\g (\rvU)) , \rvM) < \epsilon]$. Then, 
    \begin{align}
        \sup_{\f \in \setM^{\setZ}} p_e^\epsilon &\leq  1- \frac{\ent(\rvJ \mid \rvU)-\ent_2(B)}{\log(N(\epsilon)-1)}
        =     1-    \frac{\overbrace{\ent(\rvJ
        \mid \rvU,\rvC)}^{\text{irreducible ambiguity}} + \overbrace{\MI(\rvJ;\rvC \mid \rvU)}^{\text{meaning from context only}}-\ent_2(B)}{\log(N(\epsilon)-1)}. \nonumber 
    \end{align}
\end{theorem}
\begin{proof}
    See \Cref{app:more_theory2}.
\end{proof}

\Cref{thm:classification,coro:regression} tell us that speaker meaning inference from \emph{any} representation of the utterance---even that of a perfect learner trained on infinite data---cannot exceed a function of the purely linguistic ambiguity $\ent (\rvM\mid \rvU)$. $\ent (\rvM\mid \rvU)$ decomposes into a sum of two terms: $\ent(\rvM \mid \rvU,\rvC)$ or irreducible meaning ambiguity in communication, and $\MI(\rvM;\rvC\mid \rvU)$, i.e., what context resolves about meaning beyond the utterance. That is, whether communication is \emph{low}- or \emph{high}-context, governed by $\MI(\rvM;\rvC\mid \rvU)$, adjusts the ceiling on $p_e$. 

\section{Experiments}
\label{sec:experiments}

In the previous section, we put forth theoretical bounds for speaker meaning decodability from any representation of the utterance. Now, we show that these bounds hold in practice. As \Cref{thm:classification} requires access to ground truth speaker meanings, we will rely on simulated and real datasets for which intended meanings are  directly available. We test both settings: (\Cref{thm:classification}) meanings that are \emph{categorical}  and (\Cref{coro:regression}) \emph{continuous}. We first verify \Cref{thm:classification,coro:regression} in simulation using artificial languages sampled from pre-defined distributions. Then, we turn to natural language data: for categorical meanings, we consider \emph{zero-pronoun resolution} in Mandarin Chinese, and for continuous meanings, we consider a \emph{color-naming} task using \citeposs{monroe-etal-2017-colors} color reference dataset. In all cases, experiments are consistent with the theoretical bounds. We briefly sketch the task setups in this section; \textbf{for details see \Cref{app:cleaning,app:training}}.

\subsection{Artificial languages}

We start by testing \Cref{thm:classification,coro:regression} on artificial languages sampled from pre-defined distributions $\Prob(\rvM,\rvU,\rvC)$. Crucially, each language is constructed such that $\ent(\rvM \mid \rvU, \rvC)=0$, which we call the \emph{valid language assumption}: meaning can be perfectly decoded from the utterance in context, rendering communication successful. We sample a train and test dataset from $\Prob(\rvM,\rvU,\rvC)$, then retain only the utterance-meaning pairs to train a decoder for meaning inference. Note that, since utterances are not constrained to be natural language, we do not use pre-trained LLMs for the decoder: instead, we opt for MLPs, moving on to LLMs for the natural language tasks in the next section. 

\paragraph{Discrete meanings.} 
We construct artificial languages by defining $\Prob(\rvM,\rvU,\rvC)$, where $\setM$ and $\setC$ are finite sets and $\setU$ is an alphabet of symbols. \Cref{fig:panel}A illustrates this setting. In brief, in a (categorical) context $\rvC$ (\tikz\draw[draw={rgb,255:red,95;green,95;blue,211},
           fill={rgb,255:red,95;green,95;blue,211}]
  (0,0) circle (.5ex); in \Cref{fig:panel}A), the speaker intends a categorical meaning (e.g., $\blacksquare$) and produces a single symbol in $\setU$ (e.g., ``a''). The listener tries to decode back the original categorical meaning ($\blacksquare$) given the utterance (``a''). We ask whether a map can be trained to predict $\rvM$ from $\rvU$ alone.

To answer this question, we construct a family of six artificial languages.  The marginal distributions $\Prob(\rvM)$, $\Prob(\rvC)$, and $\Prob(\rvU)$ over meanings, contexts, and utterances are held uniform over $|\setM| = |\setC| = 10$, $|\setU| = 15$ values, and we systematically vary $\MI(\rvM; \rvU) \in [0,\, \log_2|\setM|]$. Lastly, on each language, we validate \Cref{thm:classification} by training an MLP to infer meanings $\rvM$ from utterances $\rvU$, and verify the best test accuracy (across hyperparameters and five random seeds) is lower than the theoretical bound. 

\begin{figure}[t]
    \centering
\includegraphics[width=\linewidth]{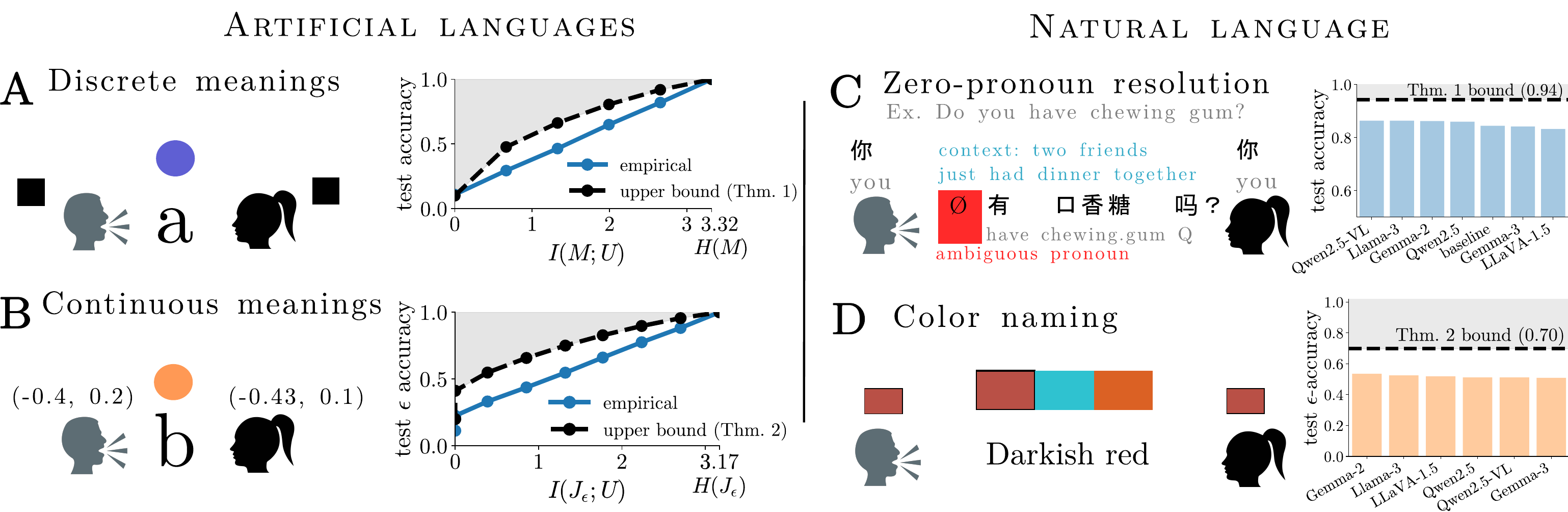}
\vspace{-4ex}
    \caption{
    \textbf{Experiments.} \textbf{A:} \emph{(Left)} Artificial languages over a categorical meaning space. In a shared categorical context $\rvC=$ \tikz\draw[draw={rgb,255:red,95;green,95;blue,211},
           fill={rgb,255:red,95;green,95;blue,211}]
  (0,0) circle (.5ex);, the speaker (left) intends meaning $\rvM=\blacksquare$ and produces utterance $\rvU=$``a''. The listener decodes $\rvU$ in context back into the original meaning. \emph{(Right)} For a family of six artificial languages (dots on the curves) that vary $\MI(\rvM;\rvU)$ (x-axis), the theoretical upper bound (black) exceeds the best empirical test accuracy across runs (blue), validating \Cref{thm:classification}. \textbf{B:} \emph{(Left)} Artificial languages over continuous meanings in the $\epsilon$-recovery setting. In shared categorical context $\rvC=$ \tikz\draw[draw={rgb,255:red,255;green,153;blue,85},fill={rgb,255:red,255;green,153;blue,85}] (0,0) circle (.5ex);, the speaker intends a 2D coordinate $\rvM \sim \textrm{Unif}[-1, 1]^2$ and produces the utterance $\rvU=$``b''. The listener approximates $\rvM$ given utterance in context. \emph{(Right)} \Cref{coro:regression} holds across eight artificial languages varying over $\MI(\rvJ; \rvU)$, seen by the theoretical bound exceeding the empirical test $\epsilon$-accuracy ($\epsilon=0.4$). \textbf{C:} \emph{(Left)} Zero-pronoun resolution in Mandarin Chinese. In a context where the speaker and listener just had dinner together, the speaker asks the listener whether she has chewing gum, dropping the pronoun (you) from the utterance. The listener infers the true pronoun in context. \emph{(Right)} For zero-pronoun resolution from LLM and VLM activations, \Cref{thm:classification} holds: the predicted accuracy ceiling ($0.93$, dashed line) exceeds the best empirical test accuracies across runs (blue bars). \textbf{D:} \emph{(Left)} Color naming in context. Given a shared context $\rvC$ of three color chips, the speaker intends the target chip $\rvM=$ \tikz\draw[draw={rgb,255:red,185;green,80;blue,70},fill={rgb,255:red,185;green,80;blue,70}] (0,0) rectangle (1.4ex,1ex); and says $\rvU=$``Darkish red'', which the listener decodes in context to click on the target. \emph{(Right)} \Cref{coro:regression}'s ceiling ($0.70$, dashed line) for $\epsilon$-accurate recovery ($\epsilon=0.2$) exceeds all models' best-case empirical performances over runs (orange bars).}
    \label{fig:panel} \vspace{-2.5ex}
\end{figure}

\paragraph{Continuous meanings.}
Here, we construct artificial languages that communicate over a \emph{continuous} meaning space. For simplicity, contexts and utterances remain categorical, with meanings drawn from a two-dimensional continuous uniform distribution $\Prob(\rvM) = \textrm{Unif}[-1,1]^2$. \Cref{fig:panel}B illustrates this setting: in a shared categorical context (\tikz\draw[draw={rgb,255:red,255;green,153;blue,85},fill={rgb,255:red,255;green,153;blue,85}] (0,0) circle (.5ex); in the figure), the speaker intends a meaning drawn from $\textrm{Unif}[-1,1]^2$ (here, $[-0.4, 0.2]$) and produces an utterance (``b''). Given the utterance, the listener approximately decodes the meaning (e.g., $[-0.43, 0.1]$). 

We construct a family of languages $\Prob(\rvM, \rvU, \rvC)$ that satisfy the valid language constraint 
($\ent(\rvM \mid \rvU, \rvC) = 0$) 
and systematically vary across $\MI(\rvJ; \rvU)$, holding all else equal.
Details on how we construct the languages are given in \Cref{app:cleaning}. Finally, to validate \Cref{coro:regression}, we train an MLP to predict $\rvM$ from $\rvU$, minimizing the MSE loss. We then compare the best-case empirical test $\epsilon$-accuracy across hyperparameter sweeps to the theoretical upper bound. 

\subsection{Natural language}

Moving on from artificial languages to natural language, we use zero-pronoun resolution in Mandarin Chinese \citep{Wang2018TranslatingPL} as a testbed for the discrete \Cref{thm:classification} and a color naming task \citep{monroe-etal-2017-colors} for the continuous \Cref{coro:regression}. Here, since we are dealing with natural language, we use pre-trained LLM activations for the utterance featurizer $\g$, and an MLP for the decoder $\f$ mapping LLM activations to the meaning. We first list the models, then sketch the tasks. 

\paragraph{Models.} For the utterance featurizer $\g: \setU \to \setZ$, we use the activations of contemporary LLMs. As a point of comparison, we also consider vision-language models (VLMs) that may also be instruction-tuned, as they are trained with extralinguistic supervision (visual and human preferences). We test six models ranging from 2B-14B parameters: three LLMs, Qwen-2.5 14B, Llama-3 8B, Gemma 2 2B, and three VLMs, Qwen-2.5 VL (7B, instruction-tuned), LLaVa 1.5 (7B), and Gemma 3 (12B, instruction-tuned). 
For each model, we aim for the best meaning inference performance possible by searching over the last-token representations\footnote{We use the last token representation as it is the only one to attend to the entire sequence.} of all layers.

\paragraph{Zero-pronoun resolution in Mandarin Chinese.}

Mandarin Chinese is a ``pronoun-drop'' (pro-drop) language: in pro-drop languages, a pronoun can be excluded from a grammatical sentence if the pronoun's identity can be inferred from surrounding context. For instance, Chinese grammar licenses both examples \textbf{(a)} and \textbf{(b)} below to mean ``do you have chewing gum?'', where, notably, \textbf{(b)} shows a ``zero'' pronoun ($\varnothing$) in place of the second-person singular \zh{你} (`you').
\vspace{-0.4ex}
\begin{exe}
\ex
\begin{minipage}[t]{0.48\linewidth}
\textbf{(a)}\label{ex:with}

\gll \zh{你} \zh{有} \zh{口香糖} \zh{吗？} \\
     nǐ yǒu kǒuxiāngtáng ma \\
     2SG have chewing.gum Q
\end{minipage}\hfill
\begin{minipage}[t]{0.48\linewidth}
\textbf{(b)}\label{ex:without}

\gll \O\ \zh{有} \zh{口香糖} \zh{吗？} \\
     \ \ yǒu kǒuxiāngtáng ma \\
     {\ } have chewing.gum Q
\end{minipage}
\end{exe}
\vspace{-1ex}
The task to infer the dropped pronoun from the utterance is known as \emph{zero-pronoun resolution}, shown in \Cref{fig:panel}C: in a context where the speaker and listener just had dinner together, the speaker asks the listener if she has chewing gum. Though the pronoun is dropped from the utterance, the listener, given context, infers the speaker asks \emph{her} for chewing gum, recovering the true pronoun ``\zh{你}'' (you).

We use zero-pronoun resolution to test \Cref{thm:classification}, repurposing \citeposs{Wang2018TranslatingPL} Chinese TV subtitles dataset.
After preprocessing (\Cref{app:cleaning}), the dataset contains $N=904{,}996$ short utterances. Of these utterances, $20.3$\% ($N=184{,}092$) drop the pronoun (as in example \textbf{(b)}).  
The context $\rvC$, including all audiovisual TV content apart from the utterance, is not provided in the dataset, but we assume that Chinese speakers recover $\rvM$ given $\rvU$ and $\rvC$ with little difficulty (valid language assumption). The final number of pronoun classes in the data is $|\setM| = 12$ (see \Cref{app:cleaning} for details); the number of possible forms is $|\setU|=13$ (possible pronouns + dropped pronoun). 

We extracted frozen LLM embeddings $\g : \setU \to \mathbb R^d$ of the utterances. Then, we trained an MLP $\f : \mathbb R^d \to \Delta^{11}$ from the LLM embedding to predict the true pronoun, minimizing the cross-entropy loss.
We compare against a strong \textbf{baseline} that directly selects the pronoun in the sentence if expressed and guesses the most frequent one \zh{我} (`I') if dropped.\footnote{Because only $20.3\%$ of the dataset showed a dropped pronoun (i.e., $79.7\%$ of the pronouns are explicit), we expect baseline performance to achieve $\approx80\%$ accuracy.} Lastly, for each model, we verify that the best-case empirical test accuracy across hyperparameters is lower than the theoretical bound.

\paragraph{\citeposs{monroe-etal-2017-colors} color naming task.}

We use \citeposs{monroe-etal-2017-colors} colors dataset as our testbed for $\epsilon$-accurate regression (\Cref{coro:regression}). The dataset is collected from humans playing a referential game \citep{Lewis1969-LEWCAP-10}. During a single round of the task, a speaker and listener are presented the same three color chips (see \Cref{fig:panel}D), where the speaker must communicate the color of a \emph{target} chip to the listener; only the speaker knows the identity of the target chip in advance. The speaker sends a message denoting the target color, e.g., ``Darkish red", to the listener, who then clicks on their guess of the target among the three chips. Communication is considered successful if the listener clicks on the correct target, given the speaker utterance and context of three chips.

We filter the dataset for communicative successes, satisfying the valid language constraint $\ent(\rvM \mid \rvU, \rvC)=0$. This yields $N=38975$ rounds of the task from various speaker-listener pairs. 
Each round contains representations for three color chips (\emph{context}) in CIELab (a 3D color space like RGB where Euclidean distances mirror human perceptual distances) \citep{gualdoni-boleda-2024-objects},
the target chip (\emph{intended meaning}), and speaker message in natural language (\emph{utterance}). 
Finally, for each model and $\epsilon$, we collect the best-case test $\epsilon$-accuracies over all runs, and verify that these lie below the theoretical upper bound in \Cref{coro:regression}. For details on estimating the upper bound, see \Cref{app:cleaning}.



\subsection{Results}
Experiments confirm the bounds on meaning inference performance predicted by \Cref{thm:classification,coro:regression}. 

\noindent \textbf{Artificial languages.} \quad \Cref{fig:panel}A (right) shows, for artificial languages whose $\MI(\rvM; \rvU)$ spans $0$ to $\log_2(|\setM|)$, the empirical test accuracy on meaning inference from utterance alone (blue curve) with the theoretical bound (dashed black curve). Two observations: first, the bound holds in practice for all tested languages (the blue curve is lower than black); second, as expected, as the mutual information between meaning and utterance increases, the theoretical bound on meaning inference performance, as well as the empirical performance, both increase. 
\Cref{fig:panel}B (right) shows a similar trend for the continuous case (for $\epsilon=0.4$; other values of $\epsilon$ can be found in \Cref{app:artificial_extra_results}). Like the discrete case, the bounds hold for all languages spanning $\MI(\rvJ; \rvU)$.

\textbf{Zero-pronoun Resolution.} \quad \Cref{fig:panel}C shows that in all models (x-axis), the best meaning inference performance (y-axis) over all hyperparameters and random seeds does not exceed the bound predicted by \Cref{thm:classification} ($0.94$, horizontal line). Interestingly, using LLMs or VLMs did not considerably improve performance beyond the baseline (accuracy $0.85$). This observation suggests that, for this dataset, the utterance alone carries little information about the dropped pronoun; LLMs and VLMs trained on web-scale data, moreover, do not overcome this bottleneck. Instead, we hypothesize that visual and social context ($\rvC$) from the TV show \citep{Wang2018TranslatingPL} enable pronoun inference.

\paragraph{Color naming.}  \Cref{fig:panel}D shows, for $\epsilon=0.2$, the predicted ceiling on meaning inference performance (dashed line); the bound holds for all models. All $\epsilon$ spanning $[0.01, 0.75]$ are shown in \Cref{fig:other_eps_colors}; the bound held in all cases, though is tighter for larger $\epsilon$, see \Cref{app:color_results} for discussion.

\section{Related Work}

A sizeable and growing body of work asks whether LLMs---their behavior or internal activations---can instantiate models or theories of language \citep{bender-koller-2020-climbing,Baroni_2022,Piantadosi_2024,boleda-2025-llms,Futrell_Mahowald_2026}. The present work was inspired by these broader debates.
\Cref{thm:classification,coro:regression} suggest that text-only LLMs should not be used to instantiate a \emph{theory} of meaning inference during naturalistic communication, given context is a crucial missing piece. However, our findings would still be consistent with using text-only LLMs' activations to \emph{model} aspects of linguistic representation that can be derived from utterance alone. 

Most similar to our work, \citet{Merrill_Goldberg_Schwartz_Smith_2021} also ask whether ungrounded LLMs are able to acquire semantics. In brief, \citet{Merrill_Goldberg_Schwartz_Smith_2021} formally show that there exist languages for which semantic equivalence (whether two utterances have the same \emph{denotation}) cannot be recovered from ungrounded linguistic input, even by a perfect learner with access to unlimited data and an ``oracle'' module providing semantic information. Our work addresses a similar high-level question and arrives at qualitatively similar conclusions, but we take a different route. In particular, (1) we consider pragmatic meaning while \citet{Merrill_Goldberg_Schwartz_Smith_2021} consider denotational meaning over sets; (2) our definition of successful meaning inference is metric-based ($\epsilon$-accurate recovery), thus gradable, while Merrill and colleagues require exact set-matching; (3) our analytic framework is \emph{information-theoretic} while theirs is based on computability; (4) we establish bounds directly on meaning decodability from a representation of text that can be empirically validated, while their results (more abstractly) establish the non-computability of semantic representation for a class of context-dependent languages. Intriguingly, the two works which diverge radically in method and conceptualization of meaning both recover (extralinguistic) grounding as the missing ingredient for meaning representation. 

Our main result (\Cref{thm:classification,coro:regression}), i.e., that text-based systems learn provably impoverished meaning representations, aligns with empirical findings that base LLMs perform worse than humans in implicature \citep{ruis_pragmatic} and presupposition \citep{sravanthi-etal-2024-pub,Chang_Bergen_2024} (see \citet{ma-etal-2025-pragmatics} for a survey). A number of positive results in pragmatic meaning inference have also been reported notwithstanding \citep{andreas-2022-language,hu-etal-2023-fine,Chang_Bergen_2024}; \Cref{thm:classification,coro:regression} suggest that high mutual information between surface forms and meanings may explain positive empirical findings in text-only LLMs' representation of pragmatic meaning.

\section{Conclusion}
We have put forth theoretical limits on meaning inference from representations of utterance alone (\Cref{thm:classification,coro:regression}). While \Cref{thm:classification,coro:regression} apply to \emph{any} representation of the utterance, we verified them for the internal activations of contemporary LLMs and VLMs, popular choices to model meaning in practice. With experiment corroborating theory, our take-home message is that there is no free lunch in meaning inference when meanings are gleaned through a linguistic channel. In natural settings, humans also make use of context as a conduit for meaning inference \citep{beuls-van-eecke-2024-humans}; however, in text-only systems, utterance forms bottleneck meaning representation, no matter whether models were trained on trillions of tokens of text or with extralinguistic supervision. 


\paragraph{Limitations} 
Our computational resources did not permit testing models larger than 14B. It is plausible that state-of-the-art LLMs would climb closer to the theoretical ceiling predicted by \Cref{thm:classification,coro:regression}. In addition, context provides a backchannel to infer meaning. An extension of this work can consider the extent to which adding context aids meaning inference, deriving similar representational limits under enriched contexts. These limits could pertain to, e.g., instruction-tuned models modified with a system prompt. Finally, we only consider meaning representation in a weak sense, i.e., statistical decodability. Future work should analyze meaning representation from a \emph{causal} lens, where decodability is a necessary first step, but not sufficient, to show whether systems trained on forms alone \emph{use} representations of meaning.




\bibliography{custom}

\makeatletter\let\addcontentsline\oldaddcontentsline\makeatother

\newpage 
\appendix
{
  \hypersetup{linkcolor=black}
  \tableofcontents
}

\newpage 

\section{Generative AI Disclosure}
Claude Code was used to clean the \citet{monroe-etal-2017-colors} colors dataset and to automate parts of the coding pipeline, especially the data pre-processing. Claude Code was used to polish the writing in \Cref{sec:setup} as well as perform mechanical tasks in the \LaTeX, e.g., ensuring uniformity in the bibliography formatting and strict usage of the \LaTeX  macros. Everything else is human-written.
Claude Code also checked the proofs of both theorems.
The authors take responsibility for all the content in this article. 

\section{Task setup}
\label{app:cleaning}

\subsection{Artificial languages: discrete meaning}

Our goal is to construct a family of distributions $\Prob(\rvM,\rvU,\rvC)$ such that they (1) are a valid language ($\ent(\rvM\mid \rvU, \rvC)=0$) and (2) we can systematically control how informative the utterance is about the meaning $\MI(\rvM;\rvU)$. For simplicity, we set context $\rvC$ to follow a uniform categorical distribution over $|\setC|=10$ categories. Similarly, meanings $\rvM$ are drawn from a (separate) uniform categorical distribution over $|\setM|=10$ categories. We set the utterances $\rvU$ to uniformly vary over $|\setU|=15$ categories; the goal is then to design the joint distribution over $\rvM$, $\rvU$, and $\rvC$ such that our constraints are satisfied.

We now design the mappings between $\rvM$, $\rvU$, and $\rvC$. Let
$p := \MI(\rvM; \rvU) / \log_2 |\setM|$ be the target
\emph{fraction} of maximum information. We split the utterance
alphabet $\setU$ into two disjoint families, corresponding to ``informative'' and ``context-dependent'' regimes. We let the probability that we are in the informative regime be realized by $p$, described below:

\begin{description}
\item[Informative utterances ($\insU \in \{0, \ldots, |\setM|-1\}$).]
      We set aside $|\setM|$ utterances that are context-\emph{independent}. These utterances correspond to the informative regime, where given $\insU$, the meaning is $\rvM = \insU$ deterministically and independently of $\insC$. Then, for every
      context $\insC$, 
      \[
          \Prob(\rvC = \insC,\; \rvM = \insU,\; \rvU = \insU)
          \;=\; \tfrac{1}{|\setC|}\;\cdot\;\tfrac{p}{|\setM|}.
      \]
      The expression on the RHS comes from multiplying $\Prob(\rvC=\insC)$ with the probability $p$ that we are in the informative regime and the probability $1 / |\setM|$ that $\rvU=\insU$ in the informative regime.
\item[Context-dependent utterances
      ($\insU \in \{|\setM|, \ldots, |\setU|-1\}$).]
      The remaining utterances are context-\emph{dependent}. Each context-dependent utterance carries $0$ bits about $\rvM$
      unconditionally but recovers $\rvM$ perfectly when combined with $\insC$. That is, given $\insU$ alone, the meaning is \emph{uniform}; given $\insU$ and
      $\insC$, the meaning is recoverable. We construct the mapping from $\insU, \insC$ to $\insM$ by setting up a deterministic cyclic shift with $\insC$ acting as a ``key'':
      \[
          \Prob(\rvC = \insC,\; \rvM = (\insU - |\setM| + \insC)\bmod|\setM|,\;
                     \rvU = \insU)
          \;=\;
          \tfrac{1}{|\setC|}\cdot
          \tfrac{1-p}{|\setU|-|\setM|}.
      \]
      For example, consider that $\setU = \{1, 2, \cdots, 15\}$ and $\setM = \{1, 2, \cdots, 10\}$. If $\insU = 11$, and $\insC = 2$, then $\insM = 3$ deterministically. The expression on the RHS is given by $\Prob(\rvC=\insC)$ multiplied by the probability $1-p$ that we're in the context-dependent regime and the probability that $\rvU = \insU$ in that regime.
\end{description}

By construction, each
utterance-in-context pair
$(\insU, \insC)$ maps to a single meaning ($\ent(\rvM \mid \rvU, \rvC) = 0$).
Also by construction, the marginal $\MI(\rvM; \rvU)$ evaluates to exactly $p\cdot \log_2 |\setM|$.
Sweeping $p$ over six evenly-spaced values in $[0, 1]$
yields the six target information levels shown on the x-axis of \Cref{fig:panel}.

Finally, to construct train and test sets, for each target mutual information, we draw $N = 10{,}000$ training examples and
$N = 5000$ disjoint sets from the joint
distribution $\Prob(\rvM, \rvU, \rvC)$.

\paragraph{Validating \Cref{thm:classification}.}
Now, we describe how to verify the bound in \Cref{thm:classification}. We realize the map $\f \circ \g : \rvU \mapsto \widehat \rvM$ by a small two layer MLP (see \Cref{app:training} for hyperparameter and training details).\footnote{NB: to avoid confusion, since we are not using LLMs in this setting, we directly use an MLP for $\f\circ \g$ (in the main text, $\g$ was defined as a featurizer and $\f$ as a decoder). In the experiments on natural language, we use LLM activations for the featurizer $\g$ and an MLP for the decoder $\f$.} First, we split the data into 80/20 train-test splits. After training this decoder to predict ground-truth meanings given utterance, we
verify that the best-case empirical test accuracy across hyperparameter settings and 5 random seeds (LHS of \Cref{thm:classification}) is lower than the theoretical upper-bound (RHS).

To estimate the RHS, we solve for the theoretical best performance $p_e$ satisfying the bound, then verify that the empirical performance is lower.
Because the joint distribution $\Prob(\rvM,\rvU,\rvC)$ was constructed analytically, we directly compute RHS terms $\ent(\rvM\mid \rvU)$ and $\log_2(|\setM|-1)$. We then find the maximum $p_e$ satisfying the upper bound
by numerical approximation, using an off-the-shelf root-finding algorithm (\texttt{scipy.optimize.brentq}). Finally, we compare this upper bound against the empirical test accuracy for each artificial language. 

\subsection{Artificial languages: continuous meaning}
\label{sec:continuous-experiment}

Our goal is to verify the $\epsilon$-accurate regression bound (\Cref{coro:regression}) in a
continuous meaning space where we can \emph{tune} the mutual
information $\MI(\rvJ; \rvU)$ between the packing-bin label
$\rvJ$ and the observed utterance $\rvU$ across the interval $[0,\, \ent(\rvJ)]$.  

We first construct $\rvJ$ given $\epsilon$. Fix $\epsilon > 0$. We first construct the $2\epsilon$-packing by randomly sampling $N=20{,}000$ i.i.d. points, and greedily selecting packing points in the order of that sample. We repeat this process $100$ times and retain the maximal $2\epsilon$-packing.
We then map meanings $\insM$ to their nearest packing point via the quantization function $q$ (described in \Cref{sec:theory}).  For $\epsilon \in \{0.2, 0.4, 0.6\}$ we obtain $N(\epsilon) := |\rvJ| \in \{25, 9, 5\}$, and
$\ent(\rvJ) = \log_2 |\rvJ| \in \{4.64, 3.17, 2.32\}$ bits (see below for why $\ent(\rvJ)$ equals $\log_2|\rvJ|$ exactly in this setup).

\paragraph{Constructing the meaning, utterance, and context distributions.} We start by defining the meanings' marginal distribution, and work with a $N=20{,}000$ sample. In short, we set this distribution to be a noisy version of \emph{the packing}: we sample
$J_i \in \{0, \ldots, N(\epsilon) - 1\}$ uniformly and set
\begin{equation}
    \rvM_i = m_{J_i} + \eta_i,
    \qquad \eta_i \sim \mathrm{Unif}\!\bigl(B_2(0, \delta)\bigr),
    \label{eq:from-packing}
\end{equation}

where $B_2(0, \delta)$ is the 2-D ball of radius $\delta=0.1$. The per-bin count is rounded so that the total $N$ is divisible by both $|\rvJ|$ and the
context alphabet size $|\setC|$; this makes bin frequencies exactly uniform
and hence $\ent(\rvJ) = \log_2 |\rvJ|$.

Now, we move on to design the joint $\Prob(\rvM, \rvU, \rvC)$. We keep contexts categorical for simplicity, fixing $|\setC| = 10$. We have $|U| = N/|\setC|=2000$.  Every meaning $\insM$ receives a unique slot $(\insU, \insC) \in \{0, \ldots, |\setU|-1\} \times \{0, \ldots, |\setC|-1\}$ from
Algorithm~\ref{alg:make-utc}, i.e., meanings are identified given utterance and context (valid language assumption, $\ent(\rvM \mid \rvU, \rvC) = 0$). The overall intuition is very similar to the discrete languages case: context $\rvC$ acts as a ``key'' that disambiguates the meaning given the utterance. The problem of tuning $\MI(\rvJ;\rvU)$ now turns into how we distribute utterance, context pairs to meanings---like the last section, we portion out ``informative'' and ``context-dependent'' utterances for every meaning.

We tune $\MI(\rvJ; \rvU)$ via
a scalar $p \in [0, 1]$ as follows:
\begin{description}
\item[\textbf{Informative utterances}] We consider the ``bins'' of meanings that the quantizer $q$ assigns to each packing point $\rvJ=j$. For each $\rvJ$-bin $j$ containing $n_j$ meanings, we take 
      $n_{\textrm{inf}} = \lfloor p\,n_j / |\setC| \rfloor \cdot |\setC|$ meanings and pack them into
      \emph{dedicated} ``$\rvU$-groups''. Each $\rvU$-group contains $|\setC|$ meanings drawn
      only from bin $j$. In these groups, the utterance perfectly identifies
      $\rvJ$ alone, i.e., the mutual information with $\rvJ$ is maximal.
\item[\textbf{Context-dependent utterances}] Then, we pool the remaining $n_{\textrm{amb}}$ meanings across all bins and assign them to
      \emph{shared} $\rvU$-groups in a round-robin way: after
      sorting the pool by $\rvJ$-bin, meaning at position
      $i$ is assigned $\insU = \insU_{\text{offset}} + (i \bmod n_{\textrm{amb}}/|\setC|)$
      and $\insC = \lfloor i / (n_{\textrm{amb}}/|\setC|) \rfloor$. Then, round-robin
      spreads each bin uniformly across all shared $\rvU$-groups, so
      $\Prob(\rvJ \mid \rvU)$ in a shared $\rvU$-group equals the
      marginal $\Prob(\rvJ)$ (up to $\pm 1$-count boundary
      effects when $|\setC|$ does not divide $n_j$).
\end{description}

By construction the informative contribution to $\MI(\rvJ; \rvU)$
is exactly $p \times \ent(
\rvJ)$ and the shared contribution is
zero. Sweeping
$p \in \{0, 1/8, 2/8, \ldots, 1\}$ therefore sweeps
$\MI(\rvJ; \rvU)$ across $[0, \ent(\rvJ)]$ while keeping the valid language constraint.

\paragraph{Validating \Cref{coro:regression}.} For each artificial language, we validate the bound. To do so, we portion the data into 80/20 train-test splits, and train a two-layer MLP to predict the continuous ground-truth meaning from the utterance, minimizing the MSE loss. The LHS of \Cref{coro:regression} is given by the test $\epsilon$-accuracy, i.e., the fraction of test datapoints predicted within-$\epsilon$ of the ground truth. For training and hyperparameter details, see \Cref{app:training}.

To estimate the upper bound on $\epsilon$-accurate recovery, we first estimate the maximum $\Prob(B)$ that satisfies \Cref{coro:regression}, where, recall, $\Prob(B)$ is the likelihood of correct classification over packing points, and $\Prob(B) \geq \textrm{\Cref{coro:regression} LHS}$ by construction. This entails estimating $\ent(\rvJ)$ via the plug-in entropy of the empirical bin
frequencies of the packing. We estimate $\MI(\rvJ; \rvU)$ using the plug-in
mutual information estimator. 
Finally, we numerically solve for the upper bound and verify that the empirical test $\epsilon$-accuracy falls below this bound.

\begin{algorithm}[t]
\caption{Build $(\rvU, \rvC)$ from $\rvJ$
  with tunable $\MI(\rvJ; \rvU) \approx p\, \ent(\rvJ)$.}
\label{alg:make-utc}
\begin{algorithmic}[1]
\Require Quantizer $q: \setM \to \mathbb N_{1\cdots |\rvJ|}$, $\rvJ$;
         context size $|\rvC|$ where $|\setC| \mid N$; parameter $p \in [0, 1]$.
\State Initialize $\rvU[\insM] = \rvC[\insM] = -1$ for all $\insM$; $\insU \gets 0$;
       $\mathcal A \gets \emptyset$.
\ForAll{$j$ in $\{0, \ldots, |\rvJ|-1\}$} \Comment{Iterate over packing bins}
  \State $\text{bin}_j \gets$ random permutation of
         $\{\insM : q(\insM) = j\}$
  \State $n^{\mathrm{inf}}_j \gets \lfloor p\, |\text{bin}_j| / |\rvC| \rfloor \cdot |\rvC|$
  \For{$s = 0, |\setC|, 2|\setC|, \ldots, n^{\mathrm{inf}}_j - |\setC|$}
    \For{$c = 0, \ldots, |\setC|-1$}
      \State $U[\text{bin}_j[s + c]] \gets u$;
             $C[\text{bin}_j[s + c]] \gets c$
    \EndFor
    \State $u \gets u + 1$
  \EndFor
  \State Append $\text{bin}_j[n^{\mathrm{inf}}_j :]$ to $\mathcal A$
\EndFor
\State $|\mathcal A| = n_{\text{amb}}$; $n^U_{\text{amb}} \gets n_{\text{amb}} / |\setC|$
\For{$i = 0, \ldots, n_{\text{amb}} - 1$}
  \Comment{Round-robin: spreads each bin across all shared U-groups}
  \State $U[\mathcal A[i]] \gets u + (i \bmod n^U_{\text{amb}})$;
         $C[\mathcal A[i]] \gets \lfloor i / n^U_{\text{amb}} \rfloor$
\EndFor
\State \Return $U, C$
\end{algorithmic}
\end{algorithm}



\subsection{Mandarin pro-drop dataset \citep{Wang2018TranslatingPL} }

\paragraph{Meaning, context, and utterance distribution.} We repurpose a dataset of $N=2{,}150{,}945$ short Mandarin TV subtitles from \citet{Wang2018TranslatingPL}. We start by filtering the dataset to utterances containing exactly one (explicit or dropped) pronoun, after which there are $N=904{,}996$ utterances. We then construct a dataset of (utterance, true pronoun) pairs (the true pronouns are already provided by \citet{Wang2018TranslatingPL}). The meaning inference pipeline (consisting of frozen LLM representation $+$ small MLP decoder) is trained to predict the true pronoun from the utterance.

We considered 12 personal pronouns in Mandarin Chinese. Singular pronouns include ``\zh{我}'' (`I'), ``\zh{你}'' (`you'), ``\zh{您}'' (`you', honorific), \zh{他} (`he'), \zh{她} (`she'), \zh{它} (`it'), along with their plural forms (singular form +\zh{们}). We normalized each possessive pronoun (base pronoun + \zh{的}) to their base singular or plural form. 
The more colloquial inclusive first-person plural \zh{咱们} (`we', which explicitly includes the interlocutor), was also present in the dataset. 
The distribution of true pronouns is skewed towards the first-person singular (37.6\%), see \Cref{fig:pronouns_distribution}; in contrast, there were no formal second-person plurals (\zh{您们}) in the dataset, which tend to be rare. This yielded 12 final pronouns.

\paragraph{Validating \Cref{thm:classification}.}
For each speaker utterance, we extract its frozen LLM embedding $\g : \setU \to \mathbb R^d; \insU \mapsto \insZ$. Then, we learn a map $\f : \mathbb R^d \to \Delta^{11}$ from the LLM embedding to a probability distribution over the 12 possible pronouns. The decoder $\f$ is an MLP of up to two layers trained to minimize the cross-entropy loss between the ground truth pronoun and the inferred one. At test time, the inferred pronoun is taken to be the most likely class under $\f$'s probability distribution. We include a strong \textbf{baseline} that directly selects the pronoun in the sentence if expressed and guesses the most frequent one \zh{我} (`I', first-person) if dropped. Because only $20.3\%$ of the dataset showed a dropped pronoun (i.e., $79.7\%$ of the pronouns are explicit), we expect baseline performance to achieve $\approx80\%$ accuracy. For each model, we estimate \Cref{thm:classification}'s LHS by taking the empirical test accuracy, and verifying that it is lower than the theoretical upper bound.

We estimate the theoretical upper bound (RHS of \Cref{thm:classification}) using the raw meanings and utterances data. Finally, like in previous tasks, we numerically solve for the largest $p_e$ that satisfies the bound and ensure that this $p_e$ is greater than the best empirical test accuracies across hyperparameters and random seeds.

\begin{figure}
\centering 
\includegraphics[width=\linewidth]{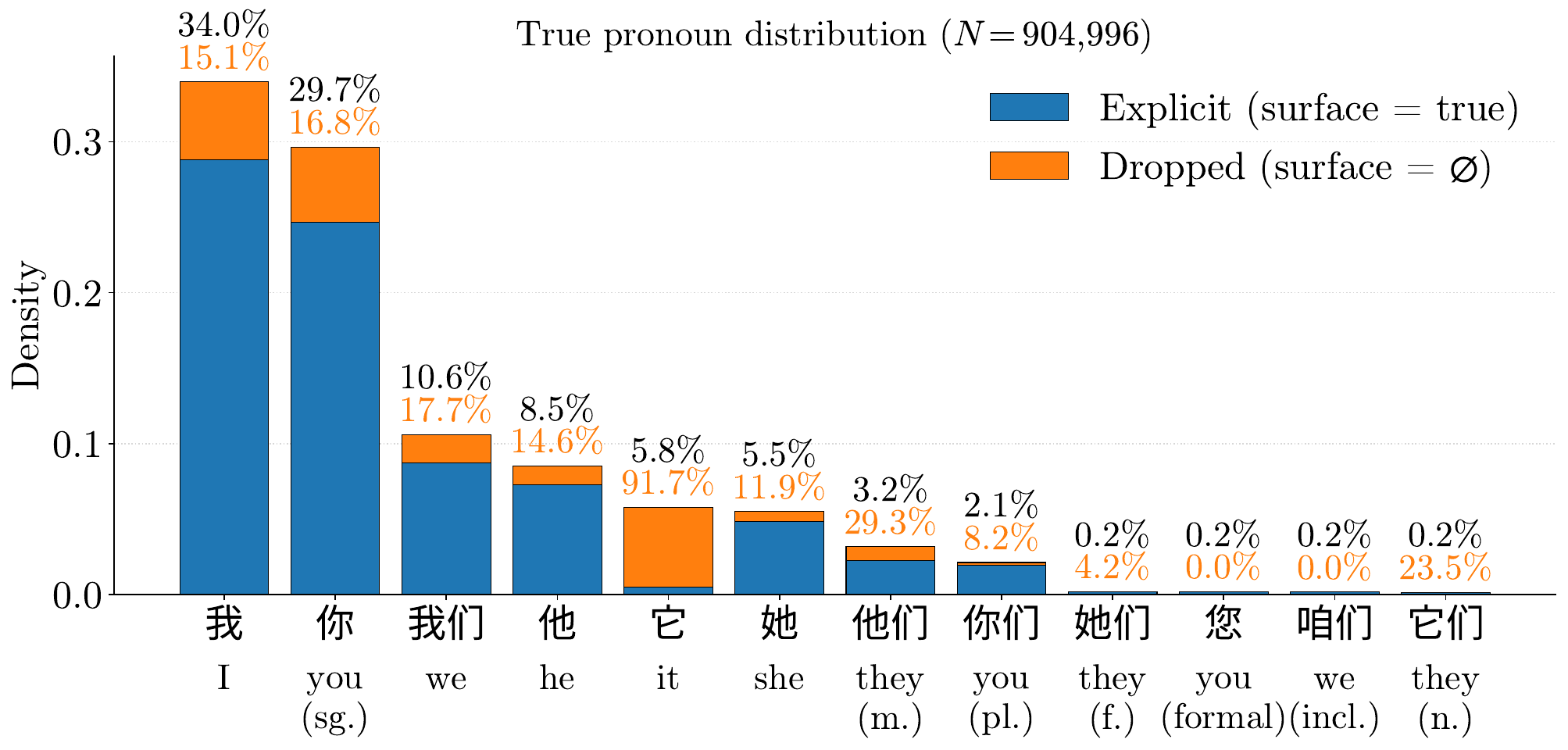}
\caption{\textbf{Mandarin pronouns distribution preprocessed from \citet{Wang2018TranslatingPL}.} Percentage annotations (black) are each pronoun's percentage of the total. The orange percentage annotations are the percent of each pronoun that is dropped (corresponding to the orange segment in the bars). For example, \zh{我} (`I') is 34\% of all true pronouns; it is dropped 15.1\% of the time.}
\label{fig:pronouns_distribution}
\end{figure}



\subsection{Color Naming \citep{monroe-etal-2017-colors}}

\paragraph{Meaning, context, and utterance distribution.} We filter the dataset for communicative successes (see \Cref{app:cleaning}), such that our valid language requirement $\ent(\rvM \mid \rvU, \rvC)=0$ is met. This cleaned version contains $N=38975$ rounds of the task from different speaker-listener pairs. Each round contains the RGB codes for the three color chips (\emph{context}), the target chip (\emph{intended meaning}), speaker message in natural language (\emph{utterance}), and listener guess (\emph{communicative success}). 

We preprocess the color chips by converting RGB to the three-dimensional CIELAB space, commonly used to model human color perception  \citep{gualdoni-boleda-2024-objects} where Euclidean distances correlate to human perceptual distances; we use this CIELAB space to proxy the \emph{meaning space} $\setM$, with the metric $d$ given by the Euclidean distance. Then, each color chip, i.e., each meaning $\insM$, is represented as a vector in $\mathbb R^3$, which we further normalized to $[0,1]^3$. Note that meanings here live in a bounded metric space as required by \Cref{coro:regression}.

\paragraph{Validating \Cref{coro:regression}.} For regression, the LHS of \Cref{coro:regression} is the empirical fraction of meanings that were predicted within-$\epsilon$. $\epsilon$ is a tolerance parameter set by the practitioner; here, we test values of $\epsilon$ ranging in $[0.01, 1]$. Similar to the zero-pronoun resolution task, we compute the empirical $\epsilon$-accuracy by using LLM representations of the speaker utterance to infer the target chip color. After analogously extracting the LLM  representations of the utterances, we learn an MLP $\f : \mathbb R^d \to \mathbb R^3; \insZ \mapsto \hat \insM$ from the LLM activations to meaning space. The map $\f$ is then trained to minimize MSE loss between the ground truth meaning $\insM$ and the inferred meaning $\hat \insM = \f (\g (\insU_n))$.
Finally, given error tolerance $\epsilon$, we report the fraction of test samples predicted within-$\epsilon$ of the target.

In contrast to the zero-pronoun resolution task, the set of possible referring expressions is unconstrained in the color naming task. Because the number of possible referring expressions can grow combinatorially with respect to utterance length, the dataset necessarily \emph{undersamples} the utterance space; this issue means that we should not directly estimate $\ent(\rvJ\mid \rvU)$ from raw utterances, as we did in the previous section. Instead, we set up an auxiliary classification task that predicts $\rvJ$ from a map $h: \setU \to \Delta^{N(\epsilon) - 1}$ that maps $\rvU$ to a probability distribution over $\rvJ$. Then, we can estimate $\ent(\rvJ\mid \rvU)$ using $\ent(\rvJ\mid h(\rvU))$, plugging the latter into the RHS of \Cref{coro:regression}.

To form $h$, we further add a classification head on top of the LLM feature vector; see \Cref{app:training} for details. On the concatenated train and test sets, we then estimate $\ent(\rvJ \mid h(\rvU))$ by taking the average cross-entropy loss of the classifier $h$. In using this strategy to estimate $\ent(\rvJ \mid h(\rvU))$, we exploit the fact that LLM features map strings to a more ``amenable'' space that encodes semantic similarities between strings. 

\section{Training details}
\label{app:training}

\paragraph{MLP hyperparameter search.}
For both \textbf{discrete} and \textbf{continuous} cases, we implement the decoder as an embedding matrix followed by a two-layer MLP, with fixed hidden dimension throughout the layers.
The MLP was trained on the following hyperparameter grid:

\begin{itemize}
    \item Learning rate $\in\{10^{-4}, 10^{-3}, 10^{-2}\}$
    \item Hidden width $\in\{64, 128, 256\}$
    \item Weight decay $\in\{0, 10^{-4}, 10^{-3}\}$.
\end{itemize}

In the \textbf{discrete} case, we trained the MLP for a maximum of 20 epochs with cross-entropy loss, batch size $256$, and Adam. In the \textbf{continuous} case, for each $(\epsilon, p)$ we train the MLP  to minimize the MSE loss with Adam, for max 30 epochs. The train/test split
is a fixed 80/20 partition of the sampled meanings. In all cases, we report the best test loss over \emph{all epochs}, hyperparameters, and five random seeds.

\paragraph{LLM hyperparameter search.} 
For every (model, layer) pair, we trained nine MLP variants that vary one
hyperparameter at a time from a shared reference point
(one hidden layer of 256 units, learning rate $10^{-3}$, weight decay
$10^{-4}$):
\begin{itemize}
\item Learning rate $\in\{10^{-4}, 10^{-3}, 10^{-2}\}$
\item Hidden width $\in\{128, 256, 512\}$
\item Weight decay $\in\{0, 10^{-4}, 10^{-3}\}$
\item MLP depth $\in\{$linear, 1-hidden, 2-hidden$\}$
\end{itemize}
For every configuration the probe is trained independently per layer. We used the Adam optimizer with early stopping
(patience 8, up to 100 epochs), and mini-batches of 256.

\paragraph{Meaning inference pipeline for natural language experiments.} 
We learn the MLP decoder $\f: \setZ \to \setM$ that maps model activations back to the original meaning space, i.e., inferring the meaning $\widehat{\rvM}$ of the utterance $\rvU$. In each experiment, we split the data, consisting of (utterance, meaning) pairs, into 80/20 train and test splits. As \Cref{thm:classification,coro:regression} are statements about the \emph{best-case} performance, for each LLM, in training the decoder $\f$, we only report the best test meaning inference performance over all LLM layers, training steps, and hyperparameters. 

\paragraph{Zero-pronoun resolution.} We drew a $50{,}000$-sentence subsample at random from the total dataset and split it 80/20 into train and test sets (the partition was identical for every model).
The original utterances contained spaces between characters, which were removed before tokenization in a preprocessing step.

For each frozen model, we extracted the last-token hidden state at every layer.
We then trained MLP probes with cross-entropy loss on the true pronoun class, sweeping the above hyperparameters, and evaluated test accuracy on
the held-out set.

\newpage 
\FloatBarrier 
\section{Additional artificial language results (continuous)}
\label{app:artificial_extra_results}
\begin{figure}[h]
    \centering
    \includegraphics[width=0.95\linewidth]{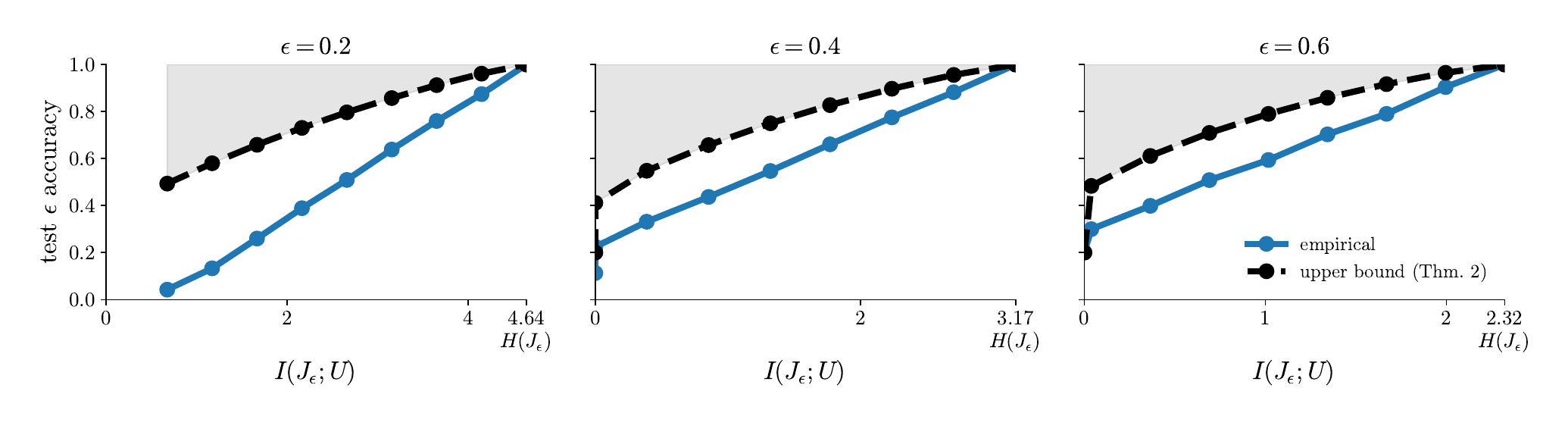}
    \caption{\textbf{Artificial languages over continuous meanings; other values of $\epsilon$}. We show that the \Cref{coro:regression} bound (black line) holds for various settings of $\epsilon\in\{0.2, 0.4, 0.6\}$.}
    \label{fig:other_eps_artificial}
\end{figure}

\FloatBarrier 

\newpage 
\FloatBarrier 
\section{Additional color-naming results}
\label{app:color_results}

\subsection{Other values of $\epsilon$}
\Cref{fig:other_eps_colors} shows the \Cref{coro:regression} in practice for other values of $\epsilon \in \{0.75, 0.5, 0.3, 0.25, 0.15, 0.125, 0.1, 0.05, 0.01\}$ on the color-naming experiment. 

\begin{figure}[h]
    \centering
    \includegraphics[width=0.75\linewidth]{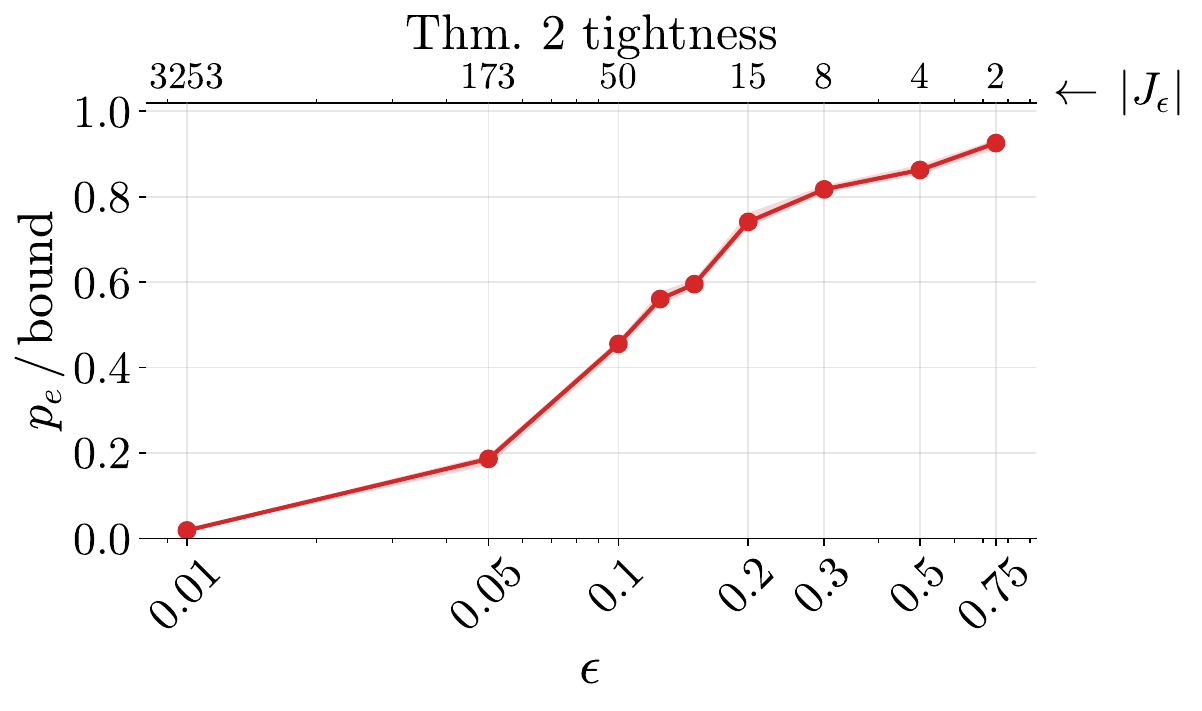}
    \caption{\textbf{\Cref{coro:regression} tightness, color-naming.} As the parameter $\epsilon$ (x-axis), which controls the granularity of the quantization, increases, the tightness of the bound in \Cref{coro:regression} (y-axis) improves. The number of bins $|\rvJ|$ in the quantization implied by $\epsilon$ is shown at the top x-axis; at the right extreme case $\epsilon=0.75$, the $\epsilon$-accurate regression task reduces to binary classification ($|\rvJ|=2$). At the left extreme, the task is $|\rvJ|=3253$-way classification. We observe that \Cref{coro:regression} becomes ``useful'' in an intermediate regime, $\epsilon \approx 0.1$.}
    \label{fig:tightness_eps}
\end{figure}

\begin{figure}
    \centering
    \includegraphics[width=\linewidth]{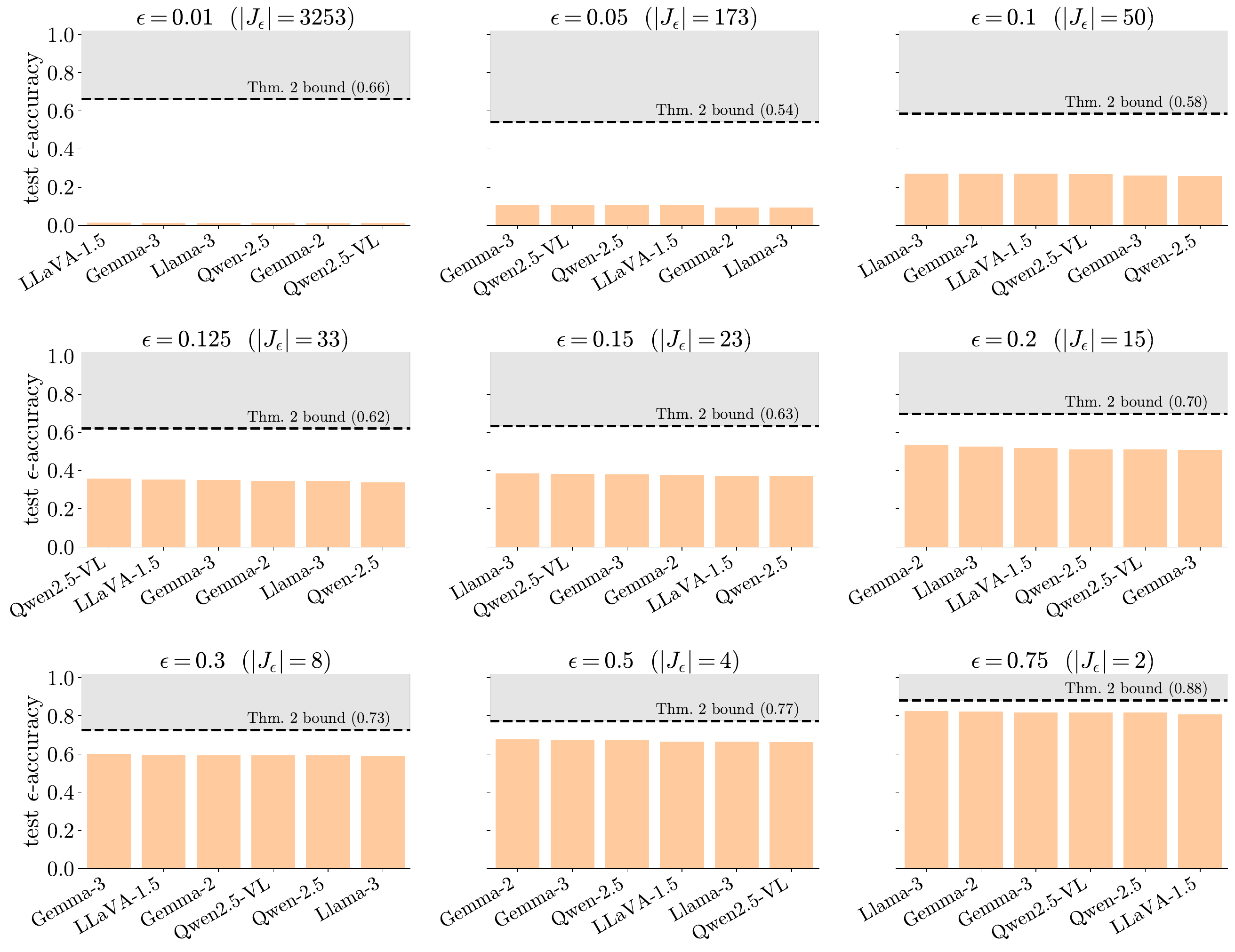}
    \caption{\textbf{Color naming, \Cref{coro:regression} for all values of $\epsilon$.} Analogue to \Cref{fig:panel} (left) for each value of $\epsilon$. Bounds hold for all values of $\epsilon$ and typically tighten with larger $\epsilon$; for reference, the number of bins $|\rvJ|$ in the quantized meaning variable $\rvJ$ is also shown. }
    \label{fig:other_eps_colors}
\end{figure}

\subsection{Tightness of \Cref{coro:regression}}

\Cref{coro:regression} holds empirically for each $\epsilon$ we tested. However, in \Cref{fig:tightness_eps} one can notice the bound is looser with smaller $\epsilon$, and especially for values of $\epsilon$ less than $0.1$ ($\epsilon=0.05$, $0.01$). We identify the source of this phenomenon as an increase in $\MI(\rvJ; \rvU)$: qualitatively, as $\epsilon$ decreases, the number of bins in the quantization increases, and a single utterance form corresponds to more bins---this results in inflated $\MI(\rvJ; \rvU)$.

Qualitatively, \emph{coarser} quantization (larger $\epsilon$) yields a better bound, while \emph{finer} quantization (smaller $\epsilon$) yields a looser bound. The loosening of the bound with $\epsilon$ can be explained by a straightforward reason: utterance forms tend to \emph{coarsely} tile the color space, and so the bounds reflect the level of inherent ambiguity in (utterance, meaning ``bin'') pairs in the quantized meaning space. To provide an intuition: utterances like ``purple'' may refer to an area of colors. If we decrease $\epsilon$ to smaller than the radius of this area, then ``purple'' is split into many bins. At test time, the classifier $h: \setU \to \rvJ$ deterministically maps ``purple'' to a single one of those bins. Because $h$ is deterministically mapping ``purple'' ($\rvU$) to the same bin ($\rvJ$), we have an increase in the mutual information $\MI(\rvJ; \rvU)$ between utterances and bin labels relative to $\ent(\rvJ)$. This results from how language coarsely parcels the meaning space.

\newpage 
\FloatBarrier 
\section{Proof of \Cref{thm:classification}}
\label{app:more_theory}

\setcounter{theorem}{0}
\begin{theorem}[Classification error]
    Let $\g \colon \setU \to \setZ$ be a map from utterances ($\setU$) to representations ($\setZ$). Abbreviate the probability that $\f$ is correct as $p_e := \Prob[\f(\g (\rvU)) = \rvM]$. Then,
        \begin{align}
        \sup_{\f \in \setM^{\setZ}} p_e &\leq 1 -\frac{\ent(\rvM \mid \rvU) - \ent_2(p_e)}{\log(|\setM|-1)} 
        =
        1- \frac{\ent(\rvM \mid \rvU,\rvC) + \overbrace{\MI(\rvM;\rvC\mid \rvU)}^{\text{meaning from context only}}-\ent_2(p_e)}{\log(|\setM|-1)}. \nonumber 
    \end{align}
\end{theorem}
\begin{proof}
    Consider the Markov chain $\rvM \to \rvZ \to \f (\rvZ)$, where $\f (\rvZ)$ is the decoder's prediction of the original speaker meaning from a representation of text, and $p_e$ is the probability of prediction success. Let $p_i = 1-p_e$ be the probability the prediction is \emph{incorrect}. By Fano's inequality, 
    \begin{align}
        \ent (\rvM\mid \rvZ) &\leq \ent_2(p_i) + p_i \log(|\setM|-1) \nonumber 
    \end{align}
    Rearranging, we have
    \begin{align}
        p_i &\geq \frac{\ent(\rvM \mid \rvZ)-\ent_2(p_i)}{\log(|\setM|-1)} \nonumber \\
        &\geq \frac{\ent(\rvM \mid \rvU)-\ent_2(p_i)}{\log(|\setM|-1)} \qquad \textcolor{gray}{(\ent(M\mid Z) \geq \ent(M\mid U))} \nonumber \\
        &= \frac{\ent(\rvM \mid \rvU,\rvC) + \MI(\rvM;\rvC\mid \rvU)-\ent_2(p_i)}{\log(|\setM|-1)} \quad \textcolor{gray}{\textrm{(\Cref{eq:mi_breakdown})}} \nonumber
    \end{align}
    \Cref{thm:classification} follows from substituting $p_i$ with $1-p_e$ in the above derivation and noting that $\ent(\rvM\mid \rvU) = \ent(\rvM \mid \rvU, \rvC) + \MI(\rvM;\rvU,\rvC)$. Due to the bounded meanings assumption $|\setM|<\infty$, the final bound is non-vacuous.
\end{proof}

\section{Extension of \Cref{thm:classification} to countably infinite meanings}
\label{app:unbounded_countable}

We extend \Cref{thm:classification} to countably infinite meaning spaces. First, we show that, for any valid probability distribution $\Prob(\rvM)$ on a countable set $\setM$, we can reasonably approximate its probability mass using a \emph{finite} subset $\setM'$ of the support $\setM$:

\begin{lemma}[Finite approximation.] 
    Let $\setM$ be countable. For all $\delta > 0$, there exists a \emph{finite} subset $\setM' \subseteq \setM$ such that $\Prob(\rvM \in \setM') \geq 1-\delta$.
\end{lemma}
\begin{proof}
    Let $p_i$ be $\Prob(\rvM=\insM_i)$, for $\insM_i \in \setM$. Since $\Prob_{\rvM}$ is a valid probability distribution, we have $\sum_{i=1}^\infty p_i = 1$. By convergence of the series, $\lim_{K\to\infty} \sum_{i=1}^K p_i = 1$. Thus, for all $\delta > 0$, there exists finite $K$ such that $\sum_{i=1}^K p_i \geq 1-\delta$. Taking $\setM' = \{\insM_1,\cdots, \insM_K\}$ yields the lemma.
\end{proof}

We have shown that $\Prob(\rvM)$, where $\rvM$ is possibly countably infinite, can be approximated using a finite subset of the support. We now use this fact to bound the classification error on countably infinite meaning spaces, by expressing it in terms of classification error on this finite subset.

\begin{corollary}[Classification error, countably infinite case] 
\label{coro:countably_infinite}
    Let $\g \colon \setU \to \setZ$ be a map from utterances ($\setU$) to representations ($\setZ$). Abbreviate the probability that $\f$ is correct as $p_e := \Prob[\f(\g (\rvU)) = \rvM]$. Then,
        \begin{align}
        \sup_{\f \in \setM^{\setZ}} p_e &\leq 1 - (1-\delta) \frac{\ent(\rvM \mid \rvU, \rvM \in \setM') - 1}{\log(|\setM'|-1)}. \nonumber 
    \end{align}
\end{corollary}
\begin{proof}
    We split $\setM$ into $\setM'$, with probability mass $q:= \Prob(\rvM \in \setM') \geq 1-\delta$, and its complement $\setM'^C = \setM \backslash \setM'$ with mass $1-q = \Prob(\rvM \in \setM'^C) \leq \delta$. Let $p_e'$ be the probability of classification accuracy on $\setM'$ by some decoder $\f \in \setM^{\setZ}$, and let $p_e'^{C}$, similarly, be the probability of classification accuracy on $\setM'^C$. We have
    \begin{align}
        p_e &= qp_e' + (1-q)p_e'^C \nonumber \\
        &\leq qp'_e + 1-q &\qquad \textrm{\textcolor{gray}{($p_e'^C \leq 1$)}} \nonumber\\ 
        &= 1-q(1-p'_e) \nonumber \\
        &\leq 1-(1-\delta)(1-p'_e) &\qquad \textrm{\textcolor{gray}{($q \geq 1-\delta$)}} \nonumber \\
        &= (1-\delta)p_e' + \delta. \nonumber
    \end{align}
    Then, 
    \begin{equation}
        \sup_{\f \in \setM^\setZ} p_e \leq (1-\delta) \sup_{\f \in \setM^\setZ} p_e' + \delta.
        \label{eq:intermediate}
    \end{equation}    
    We can bound $\sup_{\f \in \setM^\setZ}p_e'$ in the RHS of \eqref{eq:intermediate} using \Cref{thm:classification}:
    \begin{equation}
    \label{eq:intermediate2}
        \sup_{\f \in \setM^\setZ} p_e' \leq 1 - \frac{\ent(\rvM \mid \rvU, \rvM \in \setM') - \ent_2(p_e')}{\log(|\setM'|-1)} \leq 1 - \frac{\ent(\rvM \mid \rvU, \rvM \in \setM') - 1}{\log(|\setM'|-1)}.
    \end{equation}
    Finally, bounding \eqref{eq:intermediate} by the RHS of \eqref{eq:intermediate2} yields
    \begin{equation}
        \sup_{\f \in \setM^\setZ} p_e \leq 1 - (1-\delta) \frac{\ent(\rvM \mid \rvU, \rvM \in \setM') - 1}{\log(|\setM'|-1)}. \nonumber
    \end{equation}
\end{proof}

\section{Proof of \Cref{coro:regression}}
\label{app:more_theory2}
\setcounter{theorem}{1}
\begin{theorem}[$\epsilon$-accurate regression]
    Let $g: \setU \to \setZ$ be a map from utterances ($\setU$) to representations ($\setZ$). Let $(\setM, \metric)$ be a metric space and $N(\epsilon)$ be its maximal $2\epsilon$-packing number. 
    Assume $N(\epsilon) < \infty$ for $\epsilon > 0$. 
    Abbreviate the probability that $\f$ is $\epsilon$-accurate as $p_e^\epsilon = \Prob[\metric(\f(\g (\rvU)) , \rvM) < \epsilon]$. Then, 
    \begin{align}
        \sup_{\f \in \setM^{\setZ}} p_e^\epsilon &\leq  1- \frac{\ent(\rvJ \mid \rvU)-\ent_2(B)}{\log(N(\epsilon)-1)}
        =     1-    \frac{\overbrace{\ent(\rvJ
        \mid \rvU,\rvC)}^{\text{irreducible ambiguity}} + \overbrace{\MI(\rvJ;\rvC \mid \rvU)}^{\text{meaning from context only}}-\ent_2(B)}{\log(N(\epsilon)-1)}.
    \end{align}
\end{theorem}
\begin{proof}
    Let $\widehat{\rvM} = \f (\rvZ)$. Let $\{m_n\}_{n=1}^{N(\epsilon)}$ be the maximal $2\epsilon$-packing of $\setM$, meaning $\metric(m_i, m_j) \geq 2\epsilon$ for all $i\neq j \in \{1\dots N(\epsilon)\}$. Recall that $\rvJ=q(\rvM)$ is a discrete random variable over the $N(\epsilon)$ packing points $\{m_n\}_{n=1}^{N(\epsilon)}$
    , formed by mapping points $m\in\setM$ to their nearest packing point in $\{m_n\}_{n=1}^{N(\epsilon)}$. 
    We construct an auxiliary classification problem where the goal is to estimate $\widehat \rvJ = \rvJ$ given $\rvU$, the utterances associated with the packing points.

    Let the event $A$ be the event that $\epsilon$-accurate recovery is satisfied; i.e., $A = \mathbb 1[\metric(\widehat \rvM, \rvM) < \epsilon]$. Let $B$ be the event that the auxiliary classification is successful, i.e., $B = \mathbb 1[\widehat \rvJ = \rvJ]$. 
    By definition of $2\epsilon$-packing, if $\metric(\widehat \rvM, \rvM) < \epsilon$, then setting $\widehat \rvJ$ as the closest packing point to $\f(\rvZ)$ under the metric $\metric$ yields $\widehat \rvJ = \rvJ$. In other words, the event $A \Rightarrow B$, and by extension, 
    \begin{equation}
    \label{eq:events}
        \Prob(A) \leq \Prob(B).
    \end{equation}

    Note that the auxiliary classification problem, i.e., the setting of event $B$, reduces to the classification setting in \Cref{thm:classification}. We now upper-bound $\Prob(B)$ using Fano's inequality in the same way as in \Cref{thm:classification}.

    This yields
    \begin{align}
        \Prob(B)
        \leq 1-
        \frac{\ent(\rvJ\mid \rvZ)-\ent_2(B)}{\log(N(\epsilon)-1)}.
    \end{align}

    Since we assume $N(\epsilon)<\infty$, the bound is non-vacuous. 
    
    Because $\rvZ$ is a deterministic function of $\rvU$, $\ent(J \mid \rvZ) \geq \ent(J \mid \rvU).$ This yields the bound
    \begin{align}
        \Prob(B)
        \leq 1-
        \frac{\ent(\rvJ\mid \rvU)-\ent_2(B)}{\log(N(\epsilon)-1)},
    \end{align}

    which, alongside \Cref{eq:events}, yields
    \begin{align}
        \Prob(A)
        &\leq 1-
        \frac{\ent(\rvJ\mid \rvU)-\ent_2(B)}{\log(N(\epsilon)-1)}.
    \end{align}

    Finally, using
    \[
        \ent(\rvJ\mid \rvU)
        =
        \ent(\rvJ\mid \rvU,\rvC)
        +
        \MI(\rvJ;\rvC\mid \rvU)
    \]
    yields the final bound.
\end{proof}

\section{Extension of \Cref{coro:regression} to unbounded meanings}
\label{app:unbounded_uncountable}
We assume that the distribution $\Prob(\rvM)$ over meanings $\rvM$ is \emph{tight} (\Cref{sec:theory}), that is, for any $\delta > 0$, there exists a compact (closed and bounded) subset $\setM' \subset \setM$ such that $\Prob(\rvM \in \setM') \geq 1-\delta$. Let the $2\epsilon$ packing number of $\setM'$ be $N'(\epsilon)$, and construct $\rvJ$ on $\setM'$ using a quantizer function as in \Cref{coro:regression}. Then, 

\begin{corollary}[$\epsilon$-accurate recovery, unbounded case] 
\label{coro:countably_infinite}
    Let $\g \colon \setU \to \setZ$ be a map from utterances ($\setU$) to representations ($\setZ$). Let $(\setM, \metric)$ be a metric space and $N(\epsilon)$ be its maximal $2\epsilon$-packing number.  
    Abbreviate the probability that $\f$ is $\epsilon$-accurate as $p_e^\epsilon = \Prob[\metric(\f(\g (\rvU)) , \rvM) < \epsilon]$. Then, 
        \begin{align}
        \sup_{\f \in \setM^{\setZ}} p_e^\epsilon &\leq 1 - (1-\delta) \frac{\ent(\rvJ \mid \rvU, \rvM \in \setM') - 1}{\log(N'(\epsilon)-1)}. \nonumber 
    \end{align}
\end{corollary}
\begin{proof}
    Follows from (1) bounded $\setM'$ has a finite packing number $N'(\epsilon)$, (2) the proof of \Cref{coro:countably_infinite}, and (3) the proof of \Cref{coro:regression}.
\end{proof}

\end{document}